\documentclass[letterpaper]{article}
\usepackage{aaai2027}
\nocopyright 
\usepackage[hyphens]{url}
\usepackage{graphicx}
\usepackage{natbib}
\usepackage{caption}
\usepackage{algorithm}
\usepackage[noend]{algorithmic}

\usepackage{booktabs}

\usepackage{amsmath}
\usepackage{amsthm}
\usepackage{amssymb}
\usepackage{dsfont}
\usepackage{nicefrac}
\usepackage{cleveref}

\newtheorem{axiom}{Axiom}
\newtheorem{theorem}{Theorem}
\newtheorem{lemma}{Lemma}

\newtheorem{definition}{Definition}
\newtheorem{proposition}{Proposition}

\newtheorem{corollary}{Corollary}
\newtheorem{observation}{Observation}
\newtheorem{remark}{Remark}
\newtheorem{assumption}{Assumption}

\newcommand{\stocalg}{GCS}

\newtheorem{example}{Example}
\newcommand{\bft}[1]{{\textbf{#1}}}

\newenvironment{proofof}[1]{\begin{proof}[\textnormal{\textbf{Proof of \Cref{#1}}}]}{\end{proof}}

\DeclareMathOperator*{\argmax}{arg\,max}

\newcommand\ceil[1]{\left\lceil#1\right\rceil}

\newcommand{\veps}{\varepsilon}

\newcommand{\fPRA}{\mathcal F^{PRA}}

\newcommand{\appnx}[1]{{\ifnum\Includeappendix=1{#1}\else{the appendix}\fi}}

\newcount\Includeappendix
\title{Content Creation with Spillovers: An Incentive Design Approach}
\author{
    Sagi Ohayon\equalcontrib\corresponding,
    Boaz Taitler\equalcontrib,
    Omer Ben-Porat
}
\affiliations{
    Technion---Israel Institute of Technology\\
    sagio@campus.technion.ac.il, boaztaitler@campus.technion.ac.il, omerbp@technion.ac.il
}

\begin{document}

\maketitle

\begin{abstract}
The rise of AI amplifies the economic phenomenon of \emph{positive spillovers}: when creators contribute content that can be reused and adapted by LLMs, one creator's effort may improve the content quality of others through recombination. While such spillovers can improve content quality, they also reshape incentives, as creators may reduce effort when they can benefit from others' contributions. We introduce the \emph{Content Creation with Spillovers} (CCS) model, a game between a platform and strategic creators. In this game, each creator chooses an effort level, and qualities are jointly determined by all creators' efforts. The platform contracts on qualities while seeking to maximize social welfare. We show that standard contest mechanisms can be unstable, and propose a parametrized family of \emph{Provisional Allocation} mechanisms that guarantee a Pareto-dominant equilibrium. Although optimizing welfare within this family is hard to approximate in general, we develop approximation algorithms whose guarantees hold either deterministically across structured spillover classes or with high probability under random instances.
\end{abstract}

\section{Introduction}

Online content creation has become a central activity in the digital economy: individuals and firms produce text, articles, and videos to attract user attention and engagement on platforms such as YouTube and TikTok.

Because user attention and engagement are limited and often mediated by platform recommendation systems, content creators compete with one another to attract users on online platforms. A large body of literature studies content creation as a competitive environment in which creators vie for users' attention, engagement, and exposure through recommendation and ranking mechanisms that allocate users based on content performance~\cite{hron2023modeling,NEURIPS2023_2f148634,ben2020content,yao2024unveiling}. In these settings, creators strategically choose their effort levels in order to win users and increase their visibility on the platform.

More recently, the widespread adoption of AI tools has introduced a new source of interaction between content creators. Many creators now rely on large language models (LLMs) and other generative systems that are trained on, or have direct access to, vast amounts of existing online content. This creates \emph{spillovers} across creators, whereby the effort of one creator can affect the quality of content produced by others, a phenomenon well studied in economics~\cite{mas1995microeconomic}. To illustrate, consider two content creators producing material on the same platform. When one creator produces and uploads high-quality content, this content may either be incorporated into the training data of a generative model or be retrieved by the model at inference time through access to online sources. A second creator who relies on such a system to assist with writing or editing can then benefit from the first creator's effort, producing high(er)-quality content at lower cost. As a result, individual content creation decisions can affect the productivity of other content creators who rely on generative AI tools, linking creators indirectly.

While such spillovers can improve individual productivity, they can also distort incentives in competitive environments. When creators compete for user attention and exposure, a creator may benefit from reducing effort and relying on the positive spillovers generated by others, maintaining acceptable content quality at a lower cost. Anticipating this behavior, high-effort creators may, in turn, reduce their own investment, as their effort partially benefits competitors rather than translating into a competitive advantage. This strategic interaction can lead to low-effort outcomes, unstable dynamics, or the absence of equilibrium, undermining content quality.

Low-quality content is undesirable from a welfare perspective, as content quality directly affects the utility users derive from consuming content. The low-effort outcomes that can arise in the presence of spillovers therefore undermine social welfare. This raises a central design question: how should systems be designed to account for spillovers and incentivize high levels of content quality? More broadly, can spillover effects be harnessed to reinforce incentives, so that higher quality by one creator increases the incentives of others to improve their own content?

\subsection{Our Contribution}
This paper provides the first game-theoretic foundation for incentive design in AI-assisted content creation with positive spillovers. Our contribution is threefold:

\paragraph{1) Modeling Content Creation with Spillovers}
We introduce the \emph{Content Creation with Spillovers} (CCS) game between a platform and $N$ creators. Each creator chooses an effort level, $x_i$, at some cost, $c_i$. The resulting content qualities, $(Q_i(\mathbf{x}))_{i \in [N]},$ are jointly determined by \emph{all} creators' efforts--a phenomenon we call \emph{positive spillovers}.

We construct our model by analyzing quality formation in the wild, specifically through \emph{LLM scaling laws}~\cite{scalinglaws} and \emph{graph-based knowledge flows}~\cite{bramoulle2007public}, to distill a generalized assumption on spillovers. Informally, we assume the marginal quality gain from one's own effort is non-decreasing in others' efforts (see Assumption~\ref{increment plus}). We also observe that canonical mechanisms like Winner-Takes-All~\cite{hillman1989politically} and Tullock~\cite{tullock1980} are unstable: there exist game instances with no Pure Nash Equilibrium (PNE hereinafter).

\paragraph{2) The Provisional Allocation Mechanism Family}
Motivated by the instability of standard mechanisms, we propose the family of \emph{Provisional Allocation} Mechanisms $(\fPRA)$. Every mechanism $\mathcal{M}\in \fPRA$ (a PRA mechanism), is fully characterized by allocation shares $\mathbf{p} = (p_1, \ldots, p_N) \in [0,1]^N$ with $\sum_i p_i \le 1$. Under any PRA mechanism, creator $i$'s utility is $U_i(\mathbf{x}) = p_i \cdot Q_i(\mathbf{x}) - c_i(x_i)$, effectively capping creator $i$'s share at $p_i$. By separating individual rewards from the relative performance of competitors, we align creators' incentives with those of the ecosystem. We show that:
\begin{theorem}[Informal; see ~\Cref{thm:stability_selection}]
Any PRA mechanism $\mathbf{p}$ guarantees a Pareto-dominant PNE. 
\end{theorem}%
Moreover, the PRA form is also necessary: any mechanism satisfying five natural axioms must be a PRA mechanism (as we show in \appnx{Appendix~\ref{pra characterization}}).

The existence of such a Pareto-dominant PNE motivates the platform's incentive-design objective, formalized as the following optimization problem: choose $\mathbf{p}$ to maximize $SW(\overline{\mathbf{x}}(\mathbf{p}))$, where $\overline{\mathbf{x}}(\mathbf{p})$ is the Pareto-dominant PNE (see \Cref{eq:design-problem}). 

\paragraph{3) Algorithmic Results}
We study the computational complexity of the above welfare optimization problem over all PRA mechanisms, and obtain a negative result. In particular, we prove NP-hardness of finding an optimal mechanism, and show that it is hard to approximate, too (see Theorem~\ref{max q np-hard} and Corollary~\ref{no approximation}). 
On the positive side, we show that: 
\begin{theorem}[Informal; see Theorems~\ref{thm:approx_bounded},~\ref{thm random alg}]
The welfare optimization problem admits efficient approximations in the following cases:
\begin{itemize}
    \item \textbf{Bounded spillovers:} If spillovers are bounded by factors $(\beta,\eta)$ of intrinsic quality, the \textsc{NSR-Solver} algorithm (Algorithm~\ref{alg:NSR_solver}) runs in $O(N/\veps^2)$ time and returns an allocation with welfare at least $\tfrac{1}{(1+\beta)(1+\eta+N\veps)}\,OPT - 1$.
    \item \textbf{Random interaction graphs:} If the instance is drawn at random, the \textsc{Greedy Cost Selection} algorithm (Algorithm~\ref{find-portions}) runs in $O(N^2)$ time and is asymptotically optimal, achieving a $(1 - O(N^{-\nicefrac{1}{4}}))$-approximation with high probability. Simulations in \appnx{Appendix~\ref{sec:experiments}} demonstrate the algorithm's practical advantage.
    \end{itemize}
\end{theorem}
We defer formal proofs to the appendix.

\subsection{Related Work}
We divide related work into two strands: strategic content creation under AI, and public goods and spillovers.

\paragraph{Strategic content creation and AI}
The strategic behavior of content creators on algorithmically mediated platforms has received substantial attention, with a foundational line of work examining how creators compete for attention and adapt their strategies in recommendation systems~\cite{ben2019recommendation, NEURIPS2023_2f148634, boutilier2023modeling, zhu2023online}. These works address challenges such as incentivizing content quality~\cite{ghosh2011incentivizing}, improving stability of recommendation systems~\cite{adomavicius2012stability, adomavicius2015stablerecsys}, and aligning creator incentives with platform-level objectives~\cite{van2016pipelines}. More related to our work, \citet{ben2018game} seek mechanisms that satisfy several fairness and stability axioms.

The recent proliferation of Generative AI has introduced new complexity~\cite{taitler2025collaboratinggenaiincentivesreplacements}. Several works explore how AI alters the competitive dynamics between human creators and AI-generated content. \citet{yao2024human} study the competition between human creators and an AI agent, characterizing conditions under which symbiosis or conflict arises. \citet{esmaeili2024strategize} and \citet{ZHAO2026109913} examine how creators should strategize their content creation under AI--including the choice between human and AI creation modes. \citet{keinan2025strategic} introduce a model where creators strategically decide both their content quality and whether to share their content with a platform's AI system. 

\paragraph{Public goods and spillovers}
Positive spillovers are well studied in the economics literature on innovation and R\&D~\cite{cohen1989innovation, repec:nbr:nberch:8349, aghion2015knowledge}, where knowledge spillovers arise because innovations by one firm can be partially appropriated by competitors. Similar dynamics emerge in the public goods literature, where strategic agents decide how much to contribute to a shared resource~\cite{bramoulle2007public}. Recently, \citet{cheng2025networked} study networked digital public goods games with heterogeneous players, capturing the non-exclusivity of digital resources that encourages free-riding.

On the economic value of data, \citet{Jones2020} show it is driven by a "scale effect": the accumulation of information from multiple sources--rather than any single contribution. This insight motivates our social welfare objective, which aggregates quality across all creators. Empirical work on scaling laws for large language models~\cite{scalinglaws, sun2017revisiting, hestness2017deep, Rosenfeld2020a} confirms that model performance improves predictably with dataset size, providing a concrete foundation for our quality functions. However, synthetic data generated by AI systems cannot substitute for human-created content; recent studies show that training on AI-generated data leads to model collapse and degraded performance~\cite{shumailov2023curse, alemohammad2024self}. This limitation underscores the importance of incentivizing human content creation in AI ecosystems--the central challenge our work addresses.

\section{Model} \label{model wild and failure}

We study a game representing the content creation ecosystem, involving a mechanism designer (or platform) and strategic content creators. Content creators decide how much effort they want to invest in producing their content, while each creator's effort might contribute to the quality of others due to spillovers. The platform decides how to allocate user attention to them. We call this game Content Creation with Spillovers (CCS). Formally, the game is defined as a tuple $G = \langle \mathcal{N}, (Q_i)_i, (c_i)_i, \mathcal{M} \rangle$.

\paragraph{Players, Strategies and Costs} Let $\mathcal{N} = \{1, \dots, N\}$ be the set of players (content creators). Each player $i \in \mathcal{N}$ selects an \textit{effort intensity} $x_i \in [0, 1]$, where $x_i = 0$ means no effort and $x_i = 1$ means maximum capacity. This variable represents the creator's overall investment intensity, aggregating multiple inputs such as creative labor and time devoted to production. The strategy profile of all players is denoted by $\mathbf{x} = (x_1, \dots, x_N) \in [0, 1]^N$.
Exerting effort incurs a cost, modeled by a function  $c_i : [0,1] \rightarrow [0,1]$. We assume cost functions are normalized to $[0,1]$, which is standard. We impose some other standard conditions: we assume the cost function $c_i(\cdot)$ is twice continuously differentiable, convex, non-decreasing, and lastly, it satisfies $c_i(0) = 0$.

\paragraph{Quality and Spillovers}
The \emph{quality} of each player $i$'s content is determined both by the effort of $i$ and by the effort of others, a property we call \emph{spillovers}. Formally, given effort profile $\mathbf{x}$, the \textbf{Quality Function} $Q_i(\mathbf{x}) : [0,1]^N \to [0,1]$ captures how player $i$'s output quality depends on all creators' efforts, for example, through shared LLMs trained on other players' content. To that end, we assume \emph{positive spillovers}, i.e., for all players $i,j\in N$ and all profiles $\mathbf{x}$, $
\frac{\partial Q_i}{\partial x_j}(\mathbf{x}) \ge 0.
$
Note that positive spillovers are free: an increase in any player's effort weakly improves everyone's quality, while each player bears costs solely as a function of her own effort. For analytical tractability, we assume $Q_i$ is twice continuously differentiable in all arguments. We assume $Q_i$ is upper bounded for every $i$, and w.l.o.g. $Q_i \in [0, 1]$. Bounded quality represents a normalized metric, such as a predicted probability of user satisfaction, engagement, or a relative relevance grade that the platform can measure.

\paragraph{Observability and Contracting}
We assume a setting of \textit{hidden action}. While the mechanism possesses data on the ecosystem's structure, specifically knowing which creators are likely to benefit from each other's work (through the quality functions $Q(\cdot)$)—it cannot directly observe the raw effort invested by any creator. Because the platform sees only the final output quality, it cannot contract on effort directly. Thus, rewards can depend only on the quality profile $\mathbf Q = (Q_1, ..., Q_N)$.

\paragraph{Mechanisms and Utility}

The platform adopts a \emph{mechanism} that allocates attention or visibility across creators.
Formally, a mechanism is defined by a set of allocation functions $\mathcal{M} = (M_1, \dots, M_N)$, where $M_i: [0,1]^N\to [0, 1]$ determines the share of total attention captured by player $i$, as a function of $\bf Q$. We assume the total attention budget is normalized to 1, such that $\sum_{i\in\mathcal{N}}M_{i}(\bf Q)\le1$ for any profile. Note that we allow for $\sum M_{i}<1$, representing scenarios where the mechanism chooses to withhold attention (e.g., displaying AI-generated filler content) rather than allocating it to under-performing creators.
The utility of player $i$ is the value of their allocated attention minus the cost of effort:
\begin{equation}
    U_i(\mathbf{x}) = M_i(Q_1(\mathbf{x}), \dots, Q_N(\mathbf{x})) - c_i(x_i).
\end{equation}%
We define a Pure Nash Equilibrium (PNE) in the standard way \cite{Nash1950}. Namely, we say that a profile $\mathbf{x} \in [0,1]^N $ is a PNE if for every player $i \in \mathcal N$ and every unilateral deviation $x_i' \in [0,1]$, it holds that $U_i(\mathbf{x}) \geq U_i(x_i', \mathbf{x}_{-i})$, 
where $\mathbf{x}_{-i}$ denotes the strategy profiles of all players other than $i$, and $(x_i', \mathbf{x}_{-i})$ is the profile obtained by replacing the $i$th component of $\mathbf{x}$ with $x_i'$.

\paragraph{Social Welfare}
We define social welfare as the aggregate quality of content produced in the
ecosystem, given a strategy profile $\mathbf{x}$:
$SW(\mathbf{x}) = \sum_{i=1}^{N}Q_{i}(\mathbf{x})$.
This summation form captures consumer welfare directly, as higher aggregate
quality translates into higher value for users without relying on specific
platform mechanics. It is further motivated by \citet{Jones2020}, who establish
that the economic value of data is driven by a ''scale effect''--the
accumulation of information from all sources rather than any single best
contribution. While we focus on this objective throughout the paper, our results trivially extend, under minor
modifications, to other natural proxies of consumer welfare (e.g., $\sum_i M_i Q_i$).

\subsection{Positive Spillovers in the Wild}

We aspire to conduct a meaningful analysis of the ecosystem without specifying the quality functions $\mathbf{Q}$ precisely. To that end, we focus on the essential structure of spillovers. We present the following assumption:

\begin{assumption}[Effort Complementarities]
\label{increment plus}
For every player $i \in \mathcal{N}$, strategy profile $\mathbf x$ and every other player $j \in \mathcal{N}$ , $j\neq i$, it holds that $\frac{\partial^2 Q_i}{\partial x_i \partial x_j}(\mathbf{x}) \ge 0.$
\end{assumption}

Assumption~\ref{increment plus} suggests that every other player's effort weakly increases the marginal quality of player $i$'s own effort. This assumption is motivated by a large body of economics literature \cite{cohen1989innovation, repec:nbr:nberch:8349, aghion2015knowledge}.
To demonstrate that this abstract assumption captures the specific dynamics of the GenAI ecosystem, we present two concrete examples of quality formation: one based on empirical scaling laws, and one based on graph-based spillovers.

\begin{example}\label{example scaling laws}[Scaling Laws Spillovers]
Our first example is inspired by Retrieval-Augmented Generation (RAG)~\cite{lewis2020retrieval}, where new content is generated using other content as input. The quality of content generated under RAG, and how it scales with the size of the input data, is the subject of a recent body of work on scaling laws for RAG~\cite{yue2025inference, shao2024scaling}. As these works are empirical in nature and do not commit to a closed-form expression, we adopt the functional form proposed by \citet{scalinglaws}, the canonical paper of scaling laws for LLMs, in line with the broader literature~\cite{sun2017revisiting, hestness2017deep, Rosenfeld2020a}. Under this form, the loss decreases with the data volume $D$ and is well approximated by the function
$
L(D)=\left(\nicefrac{D_c}{D}\right)^{\alpha}$, where $\alpha \approx 0.095$ and $D_c>0$ is a scaling constant. Performance is therefore
$
P(D)=1-L(D)=1-\left(\nicefrac{D_c}{D}\right)^{\alpha}.
$
We interpret creators' effort as content contributed to the input data available to the GenAI system, where the total effective input is
$
D=\sum_{j=1}^N x_j + d,
$
and $d>\alpha$ represents the amount of pre-existing content already available (a mild assumption, as modern LLMs and RAG systems operate over vast corpora). The constant $D_c$ can be rescaled to match the normalization of $\mathbf{x}$ in our model.
Using this performance function, we define each player's quality as
{
\thinmuskip=1mu
\medmuskip=1mu plus 1mu minus 1mu
\thickmuskip=1mu plus 1mu
\[
Q_i(\mathbf{x})
= x_i\left(a+b\,P(\mathbf{x})\right)
= x_i\left(a+b\left[1-\left(\tfrac{D_c}{\sum_{j=1}^N x_j+d}\right)^{\alpha}\right]\right),
\]
}%
for some $a,b\ge 0$.
\end{example}

\begin{example}\label{graph example}[Graph Based Spillovers]
Our second example is inspired by the framework of \emph{Public Goods in Networks} \cite{bramoulle2007public}, and particularly the non-exclusivity of digital resources (like open-source software) encouraging free-riding \cite{cheng2025networked}. Consider a weighted interaction graph where each node $i$ possesses an intrinsic capability $q_i$, and a directed edge from $j$ to $i$ with weight $q_{ij}$ represents the spillover intensity (where $q_i, q_{ij}$ are non-negative constants). The quality function is given by:
$
Q_i(\mathbf{x}) = x_i \left( q_i + \sum_{j \neq i} q_{ij} x_j \right)
$.
This decoupled structure intuitively captures the production process: the term $x_i q_i$ reflects the quality a creator generates in isolation, while the term $x_i \sum_{j \neq i} q_{ij} x_j$ represents the additional value amplified by the ecosystem. It makes sense: one who invests more might take advantage more of the spillovers.
\end{example}

As we formally prove in \appnx{Appendix~\ref{failure appendix}},
\begin{lemma}\label{deriv. lemma}
The quality functions in Example~\ref{example scaling laws} and Example~\ref{graph example} satisfy Assumption~\ref{increment plus}.
\end{lemma}
The two examples above illustrate that Assumption~\ref{increment plus} is not merely a technical convenience, but rather captures a common structural feature of quality formation in GenAI. Consequently, from here on, we focus our attention on instances satisfying Assumption~\ref{increment plus}. 

\paragraph{Remark} Throughout the paper, we represent the spillover relationships between players, as embedded in the quality functions $(Q_i)_i$, using a directed \emph{interaction graph}. Each vertex $i$ corresponds to a player, and a directed edge $(i,j)$ indicates the presence of spillover effects from player $i$ to player $j$.

\subsection{The Platform's Design Problem}\label{subsec:design prob}
Thus far, we have deliberately abstracted away from the platform's objective. Ideally, the platform should offer high social welfare to its consumers under equilibrium. In what follows, we formally define this design problem.
Let $\mathcal G$ denote the class of all game instances with quality functions satisfying effort complementarities (Assumption~\ref{increment plus}). Further, let $\mathcal{G}(\mathcal{M})$ be the subclass of games induced by the mechanism $\mathcal{M}$. We begin by requiring mechanisms to satisfy a minimal stability property.
\begin{definition}(Stability) ~\label{stability} 
A mechanism $\mathcal{M}$ is \textit{stable} if, for every game $G \in \mathcal{G(M)}$, $G$ possesses at least one PNE.
\end{definition}
Conversely, we say a mechanism is \emph{unstable} if there exists a
game $G \in \mathcal{G}(\mathcal{M})$ that does not possess a PNE. 

The role of algorithmic stability has been noted both empirically~\cite{adomavicius2012stability, adomavicius2015stablerecsys} and theoretically~\cite{NEURIPS2023_2f148634,ben2018game}.
Another goal the designer wants to achieve, subject to being stable, is maximizing Social Welfare ($SW(\mathbf{x})$).
The optimization problem is to choose $\mathcal{M}$ that maximizes $SW$ subject to the constraint that $\mathbf{x}$ is PNE induced by $\mathcal{M}$.

\paragraph{Note on Equilibrium Selection} We acknowledge that multiple PNEs with varying welfare levels may exist, rendering the optimization constraint currently ill-defined. We formally resolve this ambiguity in the subsequent section by establishing the specific equilibrium selection.

\subsection{Failure of Popular Mechanisms} \label{sec: pre results}

Now we move to defining two standard contest formats as mechanisms in this environment, and analyzing their performance. Unfortunately, as we observe, both of them fail to possess stability.

\begin{itemize}
    \item The Winner-Takes-All mechanism, denoted $\mathcal{M}^{WTA}$. It is defined s.t. for every $i \in \mathcal N$,
    \begin{align*}
        M_i(\bf Q)=
        \begin{cases}
        \frac{1}{N(\bf Q)}  &     i\in \argmax(Q_1,...Q_N) \\ 
        0 & \textnormal{otherwise}
        \end{cases}.
    \end{align*}%
and $N(Q) = |\argmax(Q_1,\dots,Q_N)|$.
    
    \item The Tullock mechanism~\cite{tullock1980}, denoted by $\mathcal{M}^{Tull},$ is defined s.t. 
    
    $\forall i : M_i(\mathbf{Q}) = Tullock_i(Q_1,\dots,Q_N)=\frac{Q_i}{\sum_j Q_j}$.\footnote{For completeness, we assume that whenever $\bf x$ generates qualities $Q_1(\bf x)= \cdots = Q_N(\bf x)=0$, Tullock Mechanism reduces to uniform allocation, following \cite{chowdhury2011generalized}.}
\end{itemize}

\begin{proposition} \label{tullock simple instances no PNE}
$\mathcal{M}^{WTA}$ and $\mathcal{M}^{Tull}$ are unstable.
\end{proposition}



This motivates the search for stable mechanisms which maximize social welfare value.
\section{Provisional Allocation Mechanisms} \label{sec: mechanism suggestion}

As established in \Cref{tullock simple instances no PNE}, popular contest mechanisms often fail to produce stable equilibria in the presence of positive spillovers. In this section, we propose a family of mechanisms designed to handle these spillovers robustly and establish the computational complexity of optimizing $SW$ over this family.

\subsection{Defining the Family}

We start with a formal definition of the family of Provisional Allocation mechanisms

\begin{definition}
    We say $\mathcal{M}$ is a \emph{Provisional Allocation (PRA) Mechanism} if there exists a vector $\mathbf p \in [0,1]^N$ satisfying $\sum_{i=1}^N p_i \leq 1$,  such that, 
    \[
    M_i(\bf{Q}(\bf x); \bf{p}) = p_i \cdot Q_i(\mathbf{x}) \qquad \forall i \in [N].
    \]
\end{definition}%

We let $\fPRA$ be the family of all provisional allocation mechanisms. 
Every member in the family has two critical features.
First, allocation $M_i$ depends only on player~$i$'s own quality $Q_i$ and her fixed share $p_i$. A competitor's effort therefore affects player~$i$ only through the positive spillover in $Q_i(\bft{x})$, not through a reduction in her share. Second, since $Q_i \in [0,1]$, the term $p_i Q_i$ is a fraction of the share $p_i$, so $\sum_i p_i Q_i(\bft{x})$ may fall below $1$. Intuitively, low-quality outputs reduce realized visibility. The remaining attention is absorbed by baseline or AI-generated filler content,  rather than reallocating it to competitors.

\subsection{Game-Theoretic Properties}

The following proposition establishes that the Provisional Allocation family provides a robust solution to the equilibrium selection problem, guaranteeing stability and Pareto dominance of the selected equilibrium.

\begin{proposition}[Stability and Selection]
\label{thm:stability_selection}
For every mechanism $\mathcal{M} \in \mathcal{F}^{PRA}$, the following properties hold:
\begin{enumerate}
    \item \label{item:stability} \textbf{Stability:} The mechanism $\mathcal{M}$ is stable.
    \item \textbf{Equilibrium Selection:} In every induced game $\mathcal{G} \in \mathcal{G}(\mathcal{M})$, there exists a PNE, named "greatest equilibrium" and denoted by $\overline{\mathbf{x}}$, such that for any other PNE $\mathbf{x}$:
    \begin{enumerate}
        \item $SW(\overline{\mathbf{x}}) \ge SW(\mathbf{x})$. \label{item:sw}
        \item $U_i(\overline{\mathbf{x}}) \ge U_i(\mathbf{x})$ for every player $i \in \mathcal{N}$. \label{item:utility}
        \item $\overline{\mathbf{x}}$ is reachable via iterated best-response dynamics. \label{item:reachable}
    \end{enumerate}
\end{enumerate}
\end{proposition}

\Cref{thm:stability_selection} establishes equilibrium guarantees for PRA mechanisms, and its proof builds on \emph{supermodular games} \cite{topkis1998supermodularity}. In \appnx{Appendix~\ref{pra characterization}}, we complement this result with an axiomatic characterization showing that PRA mechanisms are defined by a set of desirable properties.

Since any PRA mechanism is fully characterized by its vector $\mathbf{p}$, we denote the greatest equilibrium of a game, induced by $\mathbf{p}$, as $\overline{\mathbf{x}}(\mathbf{p})$.
Consequently, \Cref{thm:stability_selection} provides a principled equilibrium selection rule, uniquely identifying the greatest equilibrium profile $\overline{\mathbf{x}}(\mathbf{p})$ as the outcome of any mechanism $\mathbf{p}$. The designer's problem becomes well-defined and ensures to yield the stable outcome on which players are most motivated to coordinate. Specifically, the optimization is over the choice of mechanism $\mathcal{M} \in \mathcal{F}^{PRA}$, or equivalently, over the allocation vector $\mathbf{p}$:
\begin{equation}
\label{eq:design-problem}
\max_{\mathbf{p}} \; SW\!\left(\overline{\mathbf{x}}(\mathbf{p})\right).
\end{equation}

\subsection{Hardness Result} \label{hardness subsection}

Having defined and analyzed the family of mechanisms we consider, we now turn to the problem of optimizing SW. We seek the mechanism parameter $\mathbf{p}$ that maximizes the greatest PNE's social welfare of the induced game. As the following results show, this optimization problem is not only computationally hard, but also hard to approximate.

\begin{theorem} \label{max q np-hard}
SW Optimization Problem~\eqref{eq:design-problem} is NP-Hard.
\end{theorem}


We prove hardness via reduction from the Max-Clique problem. Due to its inapproximability result \cite[Theorem 5.2]{Hastad1999}, we obtain that:

\begin{corollary} \label{no approximation}
For any $\veps > 0$, unless $NP=ZPP$, 
there is no polynomial-time algorithm that approximates Problem~\eqref{eq:design-problem} within a factor $N^{1-\veps}$.
\end{corollary}

Interestingly, solving Problem~\eqref{eq:design-problem} is hard even if there are no spillovers and $Q_i(\cdot)$ and $c_i(\cdot)$ are linear. In such a case, the problem is equivalent to the Knapsack problem, as shown in \appnx{Theorem~\ref{np-hard when simple}}.

\section{Efficient Algorithms for Structured Instances} \label{sec: bounded}
In this section, we focus on a fixed instance of our game and seek a vector $\mathbf{p}$ that induces a nontrivial approximation to the optimal social welfare of Problem~\eqref{eq:design-problem}. To this end, we propose two approaches. The first bounds the spillover effects, while the second focuses on instances that exhibit an underlying tree structure.
 
\subsection{Bounded Spillovers}
\label{subsec:bounded_spillovers}
We begin with a crude approach that performs well when spillovers are present but can be bounded. We first characterize bounded spillovers.
 
\begin{definition}[Bounded Spillovers]
\label{def:beta_bounded}
A quality function $Q_i$ is said to exhibit \emph{$(\beta,\eta)$-bounded spillovers} for constants $\beta,\eta \ge 0$ if for every strategy profile $\mathbf{x}$, it holds that:
\begin{enumerate}
    \item $Q_i(x_i, x_{-i}) \;\le\; (1+\beta)\, Q_i(x_i, \mathbf{0}_{-i})$
    \item $\frac{\partial Q_i}{\partial x_i}(x_i, x_{-i}) \;\le\; (1+\eta)\, \frac{\partial Q_i}{\partial x_i}(x_i, \mathbf{0}_{-i})$
\end{enumerate}
where $\mathbf{0}_{-i}$ denotes the profile in which all players other than $i$ exert zero effort.
\end{definition}
 
This definition captures the idea that while spillovers may amplify a creator's output, they are bounded by factors relative to what can be achieved in isolation. The first condition bounds the amplification of output levels, and the second bounds the amplification of the marginal product of effort.

This motivates a robust design principle: we can solve an auxiliary problem that neglects spillovers, and then use its solution in the actual game. We call this auxiliary problem the \textit{No Spillovers Relaxation (NSR)}.

 Algorithm~\ref{alg:NSR_solver} solves this relaxation. In Line~\ref{line:construct_nsr}, it constructs the NSR problem, in which each player best-responds to its allocation share as if all other players exert zero effort. Then, the objective sums the resulting qualities. In Line~\ref{line:return_opt}, it returns an optimal solution to this problem.

To ensure computational tractability, we assume a discrete framework.
We define the set $\mathcal{B}(\veps) = \{k\veps : k = 0, 1, \ldots, \lfloor 1/\veps \rfloor\}$.
In this section, we restrict our solution to allocations in which $(p_i)_i$ are multiples of $\veps$.

We now establish the guarantee of this approach with respect to the actual game's social welfare.

\begin{algorithm}[t]
\caption{No Spillovers Relaxation Solver (\textsc{NSR-Solver})}
\label{alg:NSR_solver}
\begin{algorithmic}[1]
\REQUIRE $\langle \mathcal{N}, (Q_i)_i, (c_i)_i \rangle$, $\veps > 0$
\ENSURE Provisional Allocation mechanism $\hat{\mathbf{p}}$
\STATE Construct the \emph{No Spillovers Relaxation (NSR)} problem: \label{line:construct_nsr}

\[
\begin{array}{ll}
\displaystyle \max_{\mathbf{p}\in\mathcal{B}(\varepsilon)^N}
\displaystyle \sum_{i\in\mathcal N} Q_i(y_i,\mathbf{0}_{-i}) \qquad \text{s.t.} \\[0.6em]
 \displaystyle \sum_{i\in\mathcal N} p_i \le 1\\[0.4em]
 \displaystyle
y_i \in \arg\max_{z \in [0,1]}
\{ p_i Q_i(z,\mathbf{0}_{-i}) - c_i(z) \}
\quad \forall i \in [N] 
\end{array}
\]
\RETURN an optimal solution $\hat{\mathbf{p}}$ to the problem above \label{line:return_opt}
\end{algorithmic}
\end{algorithm}

\begin{theorem}
\label{thm:approx_bounded}
Suppose all quality functions exhibit $(\beta, \eta)$-bounded spillovers, and that each $Q_i$ is concave in its own effort $x_i$. For any $\veps >0$, Algorithm~\ref{alg:NSR_solver} runs in $O(N/\veps^2)$ time\footnote{We assume the inner-function's maximizer can be computed in $O(1)$ time, which holds in many cases of interest, e.g., when this function is linear. Otherwise, it can be approximated to arbitrary precision $\delta$ in $O(\log(1/\delta))$ time.}, and outputs a valid mechanism $\hat{\mathbf{p}} \in \mathcal{B}(\varepsilon)^N$ such that:
{
\thinmuskip=1mu
\medmuskip=1mu plus 1mu minus 1mu
\thickmuskip=1mu plus 1mu
\[
SW(\overline{\mathbf{x}}(\hat{\mathbf{p}})) \;\ge\; \frac{OPT}{(1+\beta)\left(1+\min\{\eta+N \veps,\, \lceil\eta\rceil\}\right)} - 1,
 \]}%
where $OPT := \max_{\mathbf{p}\in \mathcal{B}(\varepsilon)^N,\, \sum_i p_i \le 1} SW\!\left(\overline{\mathbf{x}}(\mathbf{p})\right)$ denotes the optimal social welfare.
 \end{theorem}

The minimum lets us enjoy the better of the two terms, as the guarantee holds with both simultaneously. In particular, when $N$ is large, and a discretization with $\veps=o(1/N)$ is computationally too costly, the term $\eta+N\veps$ can exceed $\lceil\eta\rceil$. The guarantee then is more meaningful through the latter.
 
\paragraph{Why the guarantee is not trivial.}
One might expect that when spillovers are weak, ignoring them is essentially harmless. This intuition fails. The reason is that best responses need not vary smoothly with the incentives: they might behave like \emph{thresholds}, where a player exerts no effort until their allocation share crosses an activation level and then jumps to full effort (e.g., as shown later in \Cref{dense graph}). Spillovers lower these thresholds. Hence, even weak spillovers can make a set of creators jointly affordable, while it is not the case in isolation. The optimum may thus draw much of its welfare from coordinated activations. The NSR objective is blind to this coordination advantage; the essence of Theorem~\ref{thm:approx_bounded} is that the advantage is nevertheless bounded.

\begin{proof}[Proof Sketch of \Cref{thm:approx_bounded}]
The NSR problem reduces to multiple-choice knapsack, solved exactly in time $O(N/\veps^2)$, as we show in \appnx{Appendix~\ref{bounded appnx}}. For the guarantee, let $\mathbf{p}^\star$ be an optimal allocation with equilibrium $\mathbf{x}^\star$. The argument proceeds in four steps:

\emph{Step 1.} We construct an \emph{inflated} allocation $\bar{\mathbf{p}}$ by scaling $\mathbf{p}^\star$ by $1+\eta$ and rounding each entry up to the grid. Its total share is at most $1+\eta+N\veps$. So, $\bar{\mathbf{p}}$ need not satisfy $\sum_i \bar{p}_i \le 1$, but we later `extract' a valid allocation out of it.

\emph{Step 2.} Under $\bar{\mathbf{p}}$, each creator's best response, computed with all others inactive, is at least its equilibrium effort $x_i^\star$. To see why, recall that $x_i^\star$ is a best response to $p_i^\star$ given the others' equilibrium efforts, so at $x_i^\star$ the marginal quality, weighted by $p_i^\star$, covers the marginal cost. Deactivating the others lowers the marginal quality $\frac{\partial Q_i}{\partial x_i}$ by a factor of at most $1+\eta$ (the second condition of Definition~\ref{def:beta_bounded}), which the inflation of the allocation offsets. Hence, the incentive to reach $x_i^\star$ survives, and by a concavity argument the best response can only exceed it.

\emph{Step 3.} The crux is that this possibly infeasible allocation still certifies a feasible one of comparable value: as we show in \appnx{Appendix~\ref{bounded appnx}}, we can extract from $\bar{\mathbf{p}}$ a \emph{valid} allocation whose NSR objective is at least a $\tfrac{1}{1+\eta+N\veps}$ fraction of $\bar{\mathbf{p}}$'s, minus a single creator's quality. Since $\hat{\mathbf{p}}$ maximizes the NSR objective, it inherits this bound.

\emph{Step 4.} Since the game is supermodular, deploying $\hat{\mathbf{p}}$ in it yields efforts at least those computed with others inactive. So, the realized social welfare is at least the NSR objective. The first condition of Definition~\ref{def:beta_bounded}, applied at $\mathbf{x}^\star$, contributes the remaining factor $1+\beta$. 

Chaining the four steps gives the bound with the term $\eta+N\veps$. The alternative term $\lceil\eta\rceil$ follows from the same argument with $\eta$ replaced by $\lceil\eta\rceil$: since $1+\lceil\eta\rceil$ is an integer, the rounding in Step~1 is vacuous and the term $N\veps$ disappears.
\end{proof}

\subsection{Tree-Structured Instances}
\label{subsec:trees}
Beyond bounded spillovers, we also provide a solution for tree-based interaction graphs over $\veps$-discretization. When the interaction graph is a rooted tree, each creator receives spillovers from a single parent. To that end, we develop a dynamic programming-based algorithm. The algorithm exploits the recursive structure to compute an optimal allocation over $\mathcal{B}(\veps)$  in time $O(N/\veps^2)$. Due to space limitations, we defer formal statements to \appnx{Appendix~\ref{appendix:trees}}.

\section{Efficient Approximation Guarantees under Random Interaction Graphs} \label{dense graph}

In this section, we focus on the graph-based spillover structure defined in Example~\ref{graph example}. The hardness of Problem~\eqref{eq:design-problem}, established in ~\Cref{max q np-hard} and ~\Cref{no approximation}, relies on instances of this example with linear costs. We therefore ask whether such instances are also hard on average. Namely, does Problem~\eqref{eq:design-problem} admit an average-case approximation over such instances? We answer this in the affirmative.

Every instance in this section is parameterized by an intrinsic capability $q_i$, spillover intensities $g_{ij}$, interaction indicators $r_{ij} \in \{0,1\}$, and marginal cost $c_i$. The utility function of each player $i$ is then given by
$
U_i(\bft{x}) = \frac{1}{N} p_i x_i \left( q_i + \sum_{j \neq i} x_j g_{ij} r_{ij} \right) - \frac{1}{N} c_i x_i.$
This is equivalent to Example~\ref{graph example} (with linear costs), under the decomposition $q_{ij} = g_{ij} r_{ij}$, which is more convenient for specifying our following distributional assumptions.
Our analysis focuses on \emph{random interaction graphs}, as we formally define next.

\subsection{Random Interaction Graphs}

An average-case analysis requires specifying a distribution over the instances. We adopt a canonical one: the interaction pattern forms an Erd\H{o}s--R\'enyi random graph~\cite{erdds1959random}, in which player $i$ affects player $j$ with probability $r$; conditional on an interaction, the effect intensity is drawn from some distribution with expectation $\bar{q}$. We formalize this in the following definition.
\begin{definition}[Random Interaction]
\label{def:random-graph}
A \emph{random interaction} is a distribution over graph-based instances with linear costs, under which the parameters are sampled independently as follows:
$\tilde{q}_i$ and $\tilde{g}_{ij}$ are drawn from a distribution over $(0, q^\star]$ with
expectation $\bar{q}$; $\tilde{r}_{ij} \sim \mathrm{Bern}(r)$, where $r, q^\star \in (0, 1]$; and $\tilde{c}_i \sim \mathrm{Uni}([0, 1])$. $r$ and $\bar{q}$ are independent of $N$.
\end{definition}

We introduce a notation to distinguish between a random variable and its realization. We use tilde to denote a random variable, i.e.,\ a quantity induced by a fixed effort profile on the random instance parameters. For example, $\tilde{Q}_i(\bft{x}) = \frac{1}N{}x_i\bigl(\tilde{q}_i + \sum_{j \neq i} x_j \tilde{g}_{ij}\tilde{r}_{ij}\bigr)$. The same symbol without a tilde denotes a realization.

\subsection{The \text{\stocalg} Algorithm}

\begin{algorithm}[t]
\caption{Greedy Cost Selection (\stocalg)}
\label{find-portions}

\begin{algorithmic}[1]
\REQUIRE $N, (q_i)_i, (r_{ij})_{ij}, (g_{ij})_{ij}, (c_i)_i$
\ENSURE $(p_i)_i$

\STATE Relabel the players so that $c_1 \le c_2 \le \cdots \le c_N$.
\label{stocalg: sort players}

\FOR{$k \in \{N, \ldots, 1\}$}
\label{stocalg: forloop}

    \STATE $p_i \gets \dfrac{c_i}{q_i + \sum_{\substack{j = 1 \\ j \neq i}}^k r_{ij} g_{ij}} \mathds{1}_{\{ i \leq k \}}
    \quad \forall i \in [N]$
    \label{stocalg: incentivize players}

    \IF{$\sum_i p_i \leq 1$}
    \label{stocalg: check implementable}
        
        \RETURN $(p_i)_i$
        \label{stocalg: return portions}
        
    \ENDIF

\ENDFOR

\end{algorithmic}
\end{algorithm}
We now present the \text{\stocalg} algorithm, which is implemented in Algorithm~\ref{find-portions}. \text{\stocalg} is surprisingly simple and relies on the following key insight: considering player costs only (and not their qualities) suffices to approximate the optimum.  \text{\stocalg} receives all instance parameters as input and outputs a vector of allocation shares. It greedily searches for the largest set of low-cost players that can be incentivized.

After sorting players by increasing cost, it considers prefixes of this order, from largest to smallest (for loop in Line~\ref{stocalg: forloop}). For each prefix of size $k$, Line~\ref{stocalg: incentivize players} sets each selected player's allocation share to the minimum value required to cover her cost, given the spillovers generated by the other selected players, and assigns zero allocation to all remaining players. The first (and, thus, largest) allocation whose total share is at most one, i.e., satisfying the condition in Line~\ref{stocalg: return portions}, is returned.

The following result establishes the asymptotic approximation guarantee of \text{\stocalg}.
\begin{theorem} \label{thm random alg}
Let $\bft{p}^\star$ denote the optimal allocation and $\bft{p}^a$ be the output of \text{\stocalg}. There exists $C \ge 0$ such that for a large enough $N$, with probability at least $1 - 5 e^{-\frac{\sqrt{N}}{4}}$ it holds that
\[
SW(\bar{\bft{x}}(\bft{p}^a)) \geq \left(1 - C N^{-\nicefrac{1}{4}} \right) SW(\bar{\bft{x}}(\bft{p}^\star)),
\]%
and \text{\stocalg} can be implemented to run in $O(N^2)$ time.
\end{theorem}

Theorem~\ref{thm random alg} provides an asymptotic guarantee, and therefore does not by itself quantify the performance of \text{\stocalg} in the moderate-size instances most relevant in practice. Nevertheless, our simulations suggest that the phenomenon underlying the theorem is not merely asymptotic: as shown in \appnx{Appendix~\ref{sec:experiments}}, the algorithm already achieves strong approximation performance for $N \geq 100$. Next, we provide a proof sketch of the theorem.
\begin{proof}[Proof Sketch of \Cref{thm random alg}]
First, we show an interesting property of graph-based instances with linear costs: players' best responses are binary.
\begin{observation} \label{obs linear utility action}
Fix $\bft{p}$, then $\bar{x}_i(\bft{p}) \in \{0, 1\}$ $\quad \forall i \in [N]$.
\end{observation}

Next, fix an effort profile $\bft{x} \in \{ 0, 1\}^N$. Let $S(\bft{x}) = \{ i \mid x_i = 1 \}$ denote the set of effort-exerting players, and let $K(\bft{x}) = \lvert S(\bft{x}) \rvert$ denote its size. 

The feasibility condition for $\bft{p}$ induces the following feasibility condition for set $S(\bar{\bft{x}}(\bft{p}))$:
$
\sum_{i \in S(\bft{x})} \frac{c_i}{\tilde{Q}_i(\bar{\bft{x}}(\bft{p}))} \leq 1.
$

The remaining argument compares the number of effort-exerting players under GCS and under the optimal allocation. First, for a fixed profile $\bft{x}$, $\tilde{Q}_i(\bft{x})$ is a sum of $K(\bft{x})$ independent terms. Thus, when $K(\bft{x})$ is large, realized qualities are close to their expectations. Consequently, if two profiles $\bft{x}, \bft{x}'$ satisfy that $K(\bft{x}) = K(\bft{x}')$ then $SW(\bft{x}) \approx SW(\bft{x}')$.

This reduces the proof to a comparison of feasible set sizes. Since \text{\stocalg} searches over prefixes of the cost ordering, it uses the cheapest candidates for each set size. Concentration of the lowest costs and realized qualities then gives a lower bound on the number of players that \text{\stocalg} can incentivize.

To compare this with the optimum, we need a matching upper bound on the size of any feasible set. Since \text{\stocalg} already uses the cheapest candidates for each set size, another allocation could incentivize many more players only if it relied on many players whose realized qualities are much larger than their expectations. We rule this out by showing that only few such players exist.

Therefore, the number of effort-exerting players under \text{\stocalg} and under the optimal allocation differs only by a lower-order term. Since large sets with similar sizes induce similar welfare, this gives us the desired welfare guarantee.
\end{proof}

\section{Discussion and Future Work}

We have presented a novel setting that captures the economic tension of content creation in the era of GenAI. We use a game-theoretic lens to model an environment where creators exhibit positive spillovers. We proved that popular attention-allocation mechanisms fail to sustain stable equilibria in this environment. Then, we proposed a novel family of mechanisms based on provisional allocation. We proved that this approach ensures robust stability, and then we provided approximation algorithms for interesting classes of instances to overcome computational intractability of the design problem.

Our analysis relies on structural properties of the spillover functions, encapsulated in Assumption~\ref{increment plus}, which is inspired by examples of spillovers in the wild. Future research can characterize the stability guarantee without it.

Moreover, our model assumes a fixed volume of user traffic. A promising extension is to incorporate endogenous participation, allowing traffic to depend on the ecosystem's aggregated quality, as recently explored in other works \cite{keinan2025strategic, yao2024human}.

Finally, studying how different notions of fairness or diversity, such as preventing dominance by a small set of highly connected players, interact with spillovers and welfare maximization is an important direction for future work. 

\bibliography{aaai2027}

{\ifnum\Includeappendix=1{
\appendix
\onecolumn
\section{Proofs Omitted From Section~\ref{model wild and failure}} \label{failure appendix}

\begin{proofof}{deriv. lemma}
First, in Example~\ref{example scaling laws}, we verify that $Q_i$ satisfies Assumption~\ref{increment plus}.

First,
\[
\frac{\partial Q_i}{\partial x_j}
= x_i b \frac{\partial P}{\partial x_j}
= x_i b \left[ - D_c^\alpha \cdot (-\alpha)\left(\sum_{k=1}^N x_k + d\right)^{-\alpha-1} \right]
= \frac{\alpha b x_i}{\sum_{k=1}^N x_k + d}
\left(\frac{D_c}{\sum_{k=1}^N x_k + d}\right)^{\alpha}
\;\ge\; 0,
\]
So spillovers are non-negative.
\[
\begin{aligned}
\frac{\partial^2 Q_i}{\partial x_i \partial x_j}
&= \frac{\partial}{\partial x_i}
\left[ \alpha b D_c^\alpha x_i \left(\sum_{k=1}^N x_k + d\right)^{-(\alpha+1)} \right] \\
&= \alpha b D_c^\alpha
\left[ \left(\sum_{k=1}^N x_k + d\right)^{-(\alpha+1)}
+ x_i \cdot (-(\alpha+1)) \left(\sum_{k=1}^N x_k + d\right)^{-(\alpha+2)} \right] \\
&= \frac{\alpha b}{(\sum_{k=1}^N x_k + d)^2}
\left(\frac{D_c}{\sum_{k=1}^N x_k + d}\right)^{\alpha}
\left(\left(\sum_{k=1}^N x_k + d\right) - (\alpha+1)x_i\right) \\
&= \frac{\alpha b}{(\sum_{k=1}^N x_k + d)^2}
\left(\frac{D_c}{\sum_{k=1}^N x_k + d}\right)^{\alpha}
\left(\sum_{k \neq i} x_k + d - \alpha x_i\right).
\end{aligned}
\]
Since $d$ represents a pre-existing dataset, we assumed $d > \alpha$. Given that $x_i \in [0,1]$ and $\alpha=0.095$, the term $(d - \alpha x_i)$ is strictly positive even if $\sum_{k \neq i} x_k=0$. Thus,
\[
\sum_{k \neq i} x_k + d - \alpha x_i > 0 \implies \forall i,j , i\neq j : \frac{\partial^2 Q_i}{\partial x_i \partial x_j} \ge 0.
\]

Next, we do it for Example~\ref{graph example}: 
\[
\forall i,j, i\neq j: \frac{\partial Q_i}{\partial x_i \partial x_j} = q_{ij} \geq0
\]

So Assumption~\ref{increment plus} holds on the entire strategy space for both examples.
\end{proofof}

\begin{proofof}{tullock simple instances no PNE}
We prove the instability of each mechanism by providing a counterexample instance that admits no PNE. These instances are motivated by Example~\ref{graph example}.

\paragraph{Part 1: $\mathcal{M}^{WTA}$ is unstable.}
Consider an instance with $N=2$ players, where players differ in their productivity:
\[
Q_1(x)=0.5x_1,\qquad Q_2(x)=x_2,\qquad c_1(x_1)=0.25x_1,\qquad c_2(x_2)=0.25x_2 .
\]
Under the Winner-Takes-All
mechanism, player~1 wins if $Q_1>Q_2$, i.e.,\ $0.5x_1>x_2$; player~2 wins if
$x_2>0.5x_1$; and if $x_2=0.5x_1$ both players tie and each receives an allocation share
of $\tfrac12$ (as $N(Q)=2$).

Suppose toward contradiction that a PNE $(x_1^\star,x_2^\star)$ exists.
Exactly one of the following cases holds.
\begin{itemize}
\item \textbf{Case 1 (tie): $x_2^\star=0.5x_1^\star$.}
Each player receives an allocation share of $\tfrac12$, so $U_2=\tfrac12-0.25x_2^\star$.
Since $x_2^\star=0.5x_1^\star\le\tfrac12$, the deviation $x_2'=x_2^\star+\delta$
satisfies $x_2'\le1$ for any $\delta\in(0,\tfrac12]$ and is therefore feasible.
It yields $Q_2=x_2'>0.5x_1^\star=Q_1$, so player~2 wins and
obtains $U_2=1-0.25(x_2^\star+\delta)$, a strict improvement of
$\tfrac12-0.25\delta>0$. This contradicts the PNE assumption.
(In particular, this case covers the profile $(0,0)$.)

\item \textbf{Case 2 (player 1 wins): $0.5x_1^\star>x_2^\star$.}
Player~2 receives no allocation share, so if $x_2^\star>0$ she strictly improves by
deviating to $x_2'=0$ (raising her utility from $-0.25x_2^\star$ to $0$); hence
$x_2^\star=0$, which forces $x_1^\star>0$. But then player~1 can deviate to
$x_1'=x_1^\star/2$: she still wins, since $Q_1=0.25x_1^\star>0=Q_2$, and obtains
$1-0.125x_1^\star>1-0.25x_1^\star$, a contradiction.

\item \textbf{Case 3 (player 2 wins): $x_2^\star>0.5x_1^\star$.} Hence,
$x_2^\star>0$. Symmetrically, player~1 receives no share, so $x_1^\star=0$
is forced. Player~2 can then deviate to $x_2'=x_2^\star/2$: she still wins,
since $Q_2=x_2^\star/2>0=Q_1$, and obtains $1-0.125x_2^\star>1-0.25x_2^\star$,
a contradiction.
\end{itemize}
Since the three cases are exhaustive, no PNE exists.

\paragraph{Part 2: $\mathcal{M}^{Tull}$ is unstable.}

Consider an instance with $N=2$ players:
\[
Q_1(\mathbf{x}) = 0.5x_1, \qquad Q_2(\mathbf{x}) = x_1 x_2, \qquad c_1(x_1) = 0.25x_1, \qquad c_2(x_2) = 0.25x_2.
\]
Note that player 1's effort raises both her own quality and player 2's quality, while player 2's effort benefits only herself. Suppose, for contradiction, that a PNE $(x_1^\star, x_2^\star)$ exists. We consider all possible cases:

\begin{itemize}
    \item \textbf{Case 1:} $(x_1^\star, x_2^\star) = (0, 0)$.
    
    At this profile, the mechanism allocates uniformly, so $U_1(0,0) = 0.5$. However, player 1 can deviate to any small $x_1 > 0$. Since $Q_1 = 0.5x_1 > 0$ and $Q_2 = 0$, player 1 captures the full allocation
    \[
    U_1(x_1, 0) = 1 - 0.25x_1 > 0.5,
    \]
    which is a contradiction.

\item \textbf{Case 2:} $(x_1^\star, 0)$ with $x_1^\star > 0$.
    
    At this profile, $Q_1 = 0.5x_1^\star > 0$ and $Q_2 = 0$, so player 1 receives the full allocation. Player 2's utility is $U_2 = 0$. We show that player 2 has a profitable deviation.
    
    For any $x_1 > 0$, player 2's utility as a function of $x_2$ is:
    \[
    U_2(x_1, x_2) = \frac{x_1 x_2}{0.5x_1 + x_1 x_2} - 0.25x_2 = \frac{2x_2}{1 + 2x_2} - 0.25x_2.
    \]
    Taking the derivative with respect to $x_2$:
    \[
    \frac{\partial U_2}{\partial x_2} = \frac{2}{(1 + 2x_2)^2} - 0.25.
    \]
    Setting this equal to zero:
    \[
    \frac{2}{(1 + 2x_2)^2} = 0.25 \implies (1 + 2x_2)^2 = 8 \implies 1 + 2x_2 = \pm\sqrt{8},
    \]
    yielding two solutions: $x_2^{(1)} = \frac{\sqrt{8}-1}{2} \approx 0.914$ and $x_2^{(2)} = \frac{-\sqrt{8}-1}{2} \approx -1.914$. Since effort levels lie in $[0,1]$, only $x_2^{(1)} \approx 0.914$ is feasible. To verify this is a maximum, observe that $\frac{\partial U_2}{\partial x_2} > 0$ for $x_2 < 0.914$ and $\frac{\partial U_2}{\partial x_2} < 0$ for $x_2 > 0.914$. Thus, player 2's unique best response is $x_2 \approx 0.914$, which yields strictly positive utility, leading to a contradiction.
    
    \item \textbf{Case 3:} $(x_1^\star, x_2^\star)$ with $x_1^\star, x_2^\star > 0$.
    
    When both players are active, player 1's utility is:
    \[
    U_1(\mathbf{x}) = \frac{0.5x_1}{0.5x_1 + x_1 x_2} - 0.25x_1 = \frac{1}{1 + 2x_2} - 0.25x_1.
    \]
    The derivative with respect to $x_1$ is $\frac{\partial U_1}{\partial x_1} = -0.25 < 0$, so player 1's utility is strictly decreasing in her own effort. Thus, player 1 strictly prefers to deviate to a smaller effort, which contradicts our assumption. 
    \item \textbf{Case 4:} $(0, x_2^\star)$ with $x_2^\star > 0$.
    
    When $x_1 = 0$, we have $Q_2 = x_1 x_2 = 0$. Thus, player 2's quality is zero regardless of her effort, and she pays a cost of $0.25x_2^\star > 0$ for no extra allocation share. Player 2 strictly prefers to deviate to $x_2 = 0$, again leading to a contradiction.
\end{itemize}

Since all cases lead to contradiction, no PNE exists.
\end{proofof}

\section{Characterization of PRA mechanisms}  \label{pra characterization}
In this section, we characterize the PRA mechanisms using axiomatic approach. We offer $5$ axioms under which our PRA mechanism is uniquely defined.

\begin{axiom}[Best-response stability] \label{axiom: no cycles}
For every instance, there are no best-response cycles.
\end{axiom}

\begin{axiom}[Self promote] \label{axiom: self promote}
For every player $i \in [N]$ and every $\bft{Q}_{-i}$, if $q_i' > q_i$ then
\begin{align*}
\mathcal{M}_i(q_i', \bft{Q}_{-i}) > \mathcal{M}_i(q_i, \bft{Q}_{-i}).
\end{align*}
\end{axiom}

\begin{axiom}[No free-riding] \label{axiom: no free riding}
No player benefits from increase of quality of another player.
Fix $q_i$ and $\bft{Q}_{-ij}$. Then for every $q_j, q_j'$ such that $q_j' > q_j$ it holds that
\begin{align*}
\mathcal{M}_i(q_i, q_j', \bft{Q}_{-ij}) \leq \mathcal{M}_i(q_i, q_j, \bft{Q}_{-ij})
\end{align*}
\end{axiom}

\begin{axiom}[Robustness to sybil attacks] \label{axiom: no sybil attacks}
No player can benefit from splitting the quality over multiple entities. That is, fix $\bft{Q}_{-i}$, then for every $q, q'$ and $\lambda \in [0, 1]$ it holds that
\begin{align*}
&\lambda \mathcal{M}_i(q, \bft{Q}_{-i}) + (1-\lambda) \mathcal{M}_i( q', \bft{Q}_{-i}) \leq \mathcal{M}_i\left(\lambda q + (1-\lambda)q', \bft{Q}_{-i}\right).
\end{align*}
\end{axiom}

\begin{axiom}[No free meal] \label{axiom: no free meal}
For every player $i$ it holds that $\mathcal{M}_i(0, \bft{Q}_{-i}) = 0$.
\end{axiom}

\begin{theorem} \label{thm: axiom}
If mechanism $\mathcal{M}$ is twice continuously differentiable in $\bft{Q}$ and satisfies the axioms then $\mathcal{M} \in \mathcal{F}^{PRA}$.
\end{theorem}

\begin{proofof}{thm: axiom}
We begin by using the following lemma to show that for every player $i \in [N]$, there exists a function $F_i : [0, 1] \rightarrow [0, 1]$ such that $\mathcal{M}_i(\bft{Q}(\bft{x})) = F_i(Q_i)$.

\begin{lemma} \label{lemma mechanism only player}
If Axiom~\ref{axiom: no cycles}, Axiom~\ref{axiom: self promote} and Axiom~\ref{axiom: no free riding} are satisfied then for every $i, j \in [N]$ such that $i \neq j$ it holds that 
\begin{align*}
\frac{d \mathcal{M}_i}{dQ_j} = 0
\end{align*}
\end{lemma}

Next, we show that $F_i(Q_i)$ has to be linear in $Q_i$ by showing that it is both convex and concave. We start by showing it is convex.

\begin{lemma} \label{lemma mechanism is convex}
If Axiom~\ref{axiom: no cycles}, Axiom~\ref{axiom: self promote} and Axiom~\ref{axiom: no free riding} are satisfied then for every $i \in [N]$ it holds that
\begin{align*}
\frac{d^2 \mathcal{M}_i}{dQ_i^2} \geq 0.
\end{align*}
\end{lemma}

From Axiom~\ref{axiom: no sybil attacks} we get that $F_i$ has to be concave in $Q_i$. Therefore, the mechanism is of the form $\mathcal{M}_i(Q_i) = a_i Q_i + b_i$ for $a_i, b_i \in [0, 1]$.

Furthermore, from Axiom~\ref{axiom: no free meal} it holds that $F_i(0) = b_i = 0.$
Lastly, since $\sum_i \mathcal{M}_i(\bft{Q}) \leq 1$ we get that
\begin{align*}
\sum_i \mathcal{M}_i(\bft{Q}) = \sum_i a_i Q_i \leq \sum_i a_i \leq 1.
\end{align*}

denoting $p_i = a_i$ finally results in $\mathcal{M}_i(\bft{Q}) = p_i Q_i$.
This concludes the proof of \Cref{thm: axiom}.
\end{proofof}

\begin{proofof}{lemma mechanism only player}
We assume in contradiction that $\mathcal{M}$ does not have to satisfy $\frac{d \mathcal{M}_i}{dQ_j} = 0$. That is, there exists a quality profile $\bft{Q}$ such that $\frac{d \mathcal{M}_i}{dQ_j} < 0$ for some players $i,j \in [N]$. 
Our proof is constructed in 3 steps: First, we define an "adversarial" instance such that the actions of players $i$ and $j$ shift $\bft{Q}$ in a small local neighborhood where $\frac{d \mathcal{M}_i}{dQ_j} < 0$ holds. Next, we show that player $i$ chooses an action in the opposite direction to player $j$. That is, whenever $x_j = 1$, then player's $i$ best response is $x_i = 0$, and whenever $x_j = 0$ then player $i$ chooses $x_i = 1$. Lastly, we show that player $j$ chooses the action in the same direction as player $i$, i.e $x_j = 1$ when $x_i = 1$ and $x_j = 0$ when $x_i = 0$. Therefore, combining the best responses of both players results in the best-response cycle.

\paragraph{Step 1}
Without loss of generality, we focus on instances of $N = 2$. This can be seen as a special case where there are more than $2$ players, but all of them are fixed except for two.

Let $\bft{Q} = (q^0_1, q^0_2)$ such that $\frac{d \mathcal{M}_1(\bft{Q})}{dQ_2} < 0$. Then, there exists a small neighborhood, defined by $\varepsilon_1, \varepsilon_2 > 0$ such that for every $\delta_1 \in (0, \varepsilon_1]$ and $\delta_2, \delta_2' \in (0, \frac{\varepsilon_2}{2}]$, $\delta_2' > \delta_2$ it holds that
\begin{align*}
\mathcal{M}_1(q^0_1 + \delta_1, q^0_2 + \delta_2) > \mathcal{M}_1(q^0_1 + \delta_1, q^0_2 + \delta_2')
\end{align*}

We create an instance where we focus on the range $[q^0_1, q^0_1 + \varepsilon_1] \times [q^0_2, q^0_2 + \varepsilon_2]$.
Let $\delta_1 \in (0, \varepsilon_1]$, $\delta_2 \in (0, \varepsilon_2]$ and $\lambda, \tau_1, \tau_2 \in (0, 1]$. We define the qualities by
\begin{align*}
Q_1(x_1, x_2) = q^0_1 + \delta_1 x_1 \qquad Q_2(x_1, x_2) = q^0_2 + \delta_2 x_2(1 + \lambda x_1),
\end{align*}
and the costs by $c_i(x_i) = \tau_i x_i$.

\paragraph{Step 2}
Let $B_i(x_j)$ be the best response of player $i$ to the action $x_j$ of player $j$. Then our goal now is to show that $B_1(0) = 1$ and $B_1(1) = 0$. For that, we consider derivative of the utility of player $1$:
\begin{align*}
\frac{du_1}{dx_1} = \frac{d \mathcal{M}_1}{dx_1} - \frac{dc_1}{dx_1} = \frac{d \mathcal{M}_1}{dQ_1} \delta_1 + \frac{d\mathcal{M}_1}{dQ_2}\delta_2 \lambda x_2 - \tau_1.
\end{align*}
Therefore, we get that
\begin{align*}
\begin{cases}
\frac{du_1}{dx_1} = \frac{d \mathcal{M}_1(q^0_1 + \delta_1 x_1, q^0_2)}{dQ_1}\delta_1 - \tau_1 & \mbox{$x_2 = 0$} \\
\frac{du_1}{dx_1} = \frac{d \mathcal{M}_1(q^0_1 + \delta_1 x_1, q^0_2 + \delta_2 (1 + \lambda x_1))}{dQ_1}\delta_1 + \frac{d \mathcal{M}_1(q^0_1 + \delta_1 x_1, q^0_2 + \delta_2 (1 + \lambda x_1))}{dQ_2}\delta_2 \lambda - \tau_1 & \mbox{$x_2 = 1$}
\end{cases}
\end{align*}

Next, recall that we assume that $\frac{d\mathcal{M}_1}{dQ_2} < 0$. Therefore, choose $\delta_1$ small enough such that
\begin{align} \label{ineq delta small neg ux1}
\max_{x_1} \frac{d \mathcal{M}_1(q^0_1 + \delta_1 x_1, q^0_2 + \delta_2 (1 + \lambda x_1))}{dQ_1} \delta_1 < \delta_2 \lambda \min_{x_1} \left| \frac{d \mathcal{M}_1(q^0_1 + \delta_1 x_1, q^0_2 + \delta_2 (1 + \lambda x_1))}{dQ_2} \right|,
\end{align}

and let $\tau_1 = \min_{x_1} \frac{d \mathcal{M}_1(q^0_1 + \delta_1 x_1, q^0_2)}{dQ_1} \frac{\delta_1}{2}$. Therefore, we get that

\begin{align*}
\frac{du_1(x_1, 0)}{dx_1} &= \frac{d \mathcal{M}_1(q^0_1 + \delta_1 x_1, q^0_2)}{dQ_1}\delta_1 - \tau_1 \\
&= \frac{d \mathcal{M}_1(q^0_1 + \delta_1 x_1, q^0_2)}{dQ_1}\delta_1 - \min_{x_1} \frac{d \mathcal{M}_1(q^0_1 + \delta_1 x_1, q^0_2)}{dQ_1} \frac{\delta_1}{2} \\
&\geq \min_{x_1} \frac{d \mathcal{M}_1(q^0_1 + \delta_1 x_1, q^0_2)}{dQ_1} \delta_1 - \min_{x_1} \frac{d \mathcal{M}_1(q^0_1 + \delta_1 x_1, q^0_2)}{dQ_1} \frac{\delta_1}{2} \\
&= \min_{x_1} \frac{d \mathcal{M}_1(q^0_1 + \delta_1 x_1, q^0_2)}{dQ_1} \frac{\delta_1}{2} > 0,
\end{align*}

where the last inequality is due to Axiom~\ref{axiom: self promote}.
Since this is true for every $x_1$ it holds that $x_1 = 1$ satisfies that $x_i = B_1(0)$.
Next, Axiom~\ref{axiom: no free riding} and Inequality~\eqref{ineq delta small neg ux1} imply that
\begin{align*}
\max_{x_1} \frac{d \mathcal{M}_1(q^0_1 + \delta_1 x_1, q^0_2 + \delta_2 (1 + \lambda x_1))}{dQ_1} \delta_1 + \delta_2 \lambda \min_{x_1} \frac{d \mathcal{M}_1(q^0_1 + \delta_1 x_1, q^0_2 + \delta_2 (1 + \lambda x_1))}{dQ_2} < 0.
\end{align*}

Plugging that into the derivative of $u_1$ for $x_2 = 1$ results in
\begin{align*}
\frac{du_1(x_1, 1)}{dx_1} &= \frac{d \mathcal{M}_1(q^0_1 + \delta_1 x_1, q^0_2 + \delta_2 (1 + \lambda x_1))}{dQ_1}\delta_1 + \frac{d \mathcal{M}_1(q^0_1 + \delta_1 x_1, q^0_2 + \delta_2 (1 + \lambda x_1))}{dQ_2}\delta_2 \lambda - \tau_1 \\
&\leq \max_{x_1} \frac{d \mathcal{M}_1(q^0_1 + \delta_1 x_1, q^0_2 + \delta_2 (1 + \lambda x_1))}{dQ_1} \delta_1 + \delta_2 \lambda \max_{x_1} \frac{d \mathcal{M}_1(q^0_1 + \delta_1 x_1, q^0_2 + \delta_2 (1 + \lambda x_1))}{dQ_2} - \tau_1 < 0.
\end{align*}
Therefore, $B_1(1) = 0$.

\paragraph{Step 3} we now repeat the same type of arguments for player 2. Starting with the utility
\begin{align*}
\frac{du_2}{dx_2} = \frac{d\mathcal{M}_2}{dx_2} - \frac{dc_2}{dx_2} = \frac{d\mathcal{M}_2}{dQ_2} \delta_2\left(1 + \lambda x_1 \right) - \tau_2.
\end{align*}
Therefore,
\begin{align*}
\begin{cases}
\frac{du_2}{dx_2} = \frac{d\mathcal{M}_2 (q^0_1, q^0_2 + x_2 \delta_2)}{dQ_2} \delta_2  - \tau_2 & \mbox{$x_1 = 0$} \\
\frac{du_2}{dx_2} = \frac{d\mathcal{M}_2 (q^0_1 + \delta_1, q^0_2 + x_2 \delta_2 (1 + \lambda))}{dQ_2} \delta_2 (1 + \lambda)  - \tau_2 & \mbox{$x_1 = 1$}
\end{cases}
\end{align*}

Since $\frac{d\mathcal{M}}{dQ_2}$ is continuous, for every $\eta > 0$, there exists $\delta_1, \delta_2$ small enough such that for every $x_2, x_2'$ it holds that
\begin{align*}
\left| \frac{d\mathcal{M}(q^0_1, q^0_2 + x_2 \delta_2)}{dQ_2} - \frac{d\mathcal{M}_2 (q^0_1 + \delta_1, q^0_2 + x_2' \delta_2 (1 + \lambda))}{dQ_2} \right| \leq \eta.
\end{align*}

By choosing $\eta = \min_{x_1, x_2} \frac{d \mathcal{M}_2}{dQ_2} \delta_2 (1 + 0.5\lambda)$,  $\delta_1, \delta_2$ such that $\eta < 0.5 \min_{x_1, x_2} \frac{d \mathcal{M}_2}{dQ_2} \lambda$ and $\tau_2 = \min_{x_1, x_2} \frac{d \mathcal{M}_2}{dQ_2} \delta_2 (1 + 0.5\lambda)$ we get that
\begin{align*}
\frac{du_2(0, x_2)}{dx_2} &= \frac{d\mathcal{M}_2 (q^0_1, q^0_2 + x_2 \delta_2)}{dQ_2} \delta_2  - \tau_2 \\
&\leq \min_{x_1, x_2} \frac{d \mathcal{M}_2}{dQ_2} \delta_2 + \eta \delta_2 - \min_{x_1, x_2} \frac{d \mathcal{M}_2}{dQ_2} \delta_2 (1 + 0.5\lambda) \\
&= -0.5 \min_{x_1, x_2} \frac{d \mathcal{M}_2}{dQ_2} \delta_2 \lambda + \eta \delta_2 < 0.
\end{align*}

and 

\begin{align*}
\frac{du_2(1, x_2)}{dx_2} &= \frac{d\mathcal{M}_2 (q^0_1 + \delta_1, q^0_2 + x_2 \delta_2 (1 + \lambda))}{dQ_2} \delta_2 (1 + \lambda)  - \tau_2 \\
&\geq \min_{x_1, x_2} \frac{d \mathcal{M}_2}{dQ_2} \delta_2 (1 + \lambda) - \tau_2 \\
&= \min_{x_1, x_2} \frac{d \mathcal{M}_2}{dQ_2} \delta_2 (1 + \lambda)  - \min_{x_1, x_2} \frac{d \mathcal{M}_2}{dQ_2} \delta_2 (1 + 0.5\lambda) \\
&= 0.5 \min_{x_1, x_2} \frac{d \mathcal{M}_2}{dQ_2} \delta_2 \lambda > 0.
\end{align*}

Hence, we got that $0 = B_2(0)$ and $1 = B_2(1)$, which means that there is a best response cycle where the players move between the profiles: 
\begin{align*}
(0, 0) \rightarrow (1, 0) \rightarrow (1,1) \rightarrow (0, 1) \rightarrow (0, 0),
\end{align*}
which contradicts Axiom~\ref{axiom: no cycles}.

This concludes the proof of \Cref{lemma mechanism only player}.
\end{proofof}

\begin{proofof}{lemma mechanism is convex}
From \Cref{lemma mechanism only player} we get that for every player $i \in [N]$, there exists a function $F_i(Q_i)$ such that $\mathcal{M}_i(\bft{Q}) = F_i(Q_i)$.

We assume in contradiction that $\mathcal{M}$ does not have to satisfy $\frac{d^2 \mathcal{M}}{dQ_i^2} \geq 0$. That is, there exists a quality profile $\bft{Q}$ such that $\frac{d^2 \mathcal{M}_i(\bft{Q})}{dQ_i^2} < 0$ for some player $i \in [N]$. 
Our proof is constructed in 3 steps: First, we define an "adversarial" instance such that the actions of players $i$ and $j$ shift $\bft{Q}$ in a small local neighborhood where $\frac{d^2 \mathcal{M}(\bft{Q})}{dQ_i^2} < 0$ holds. Next, we show that player $i$ chooses an action in the opposite direction to player $j$. That is, whenever $x_j = 1$, then player's $i$ best response is $x_i = 0$, and whenever $x_j = 0$ then player $i$ chooses $x_i = 1$. Lastly, we show that player $j$ chooses the action in the same direction as player $i$, i.e $x_j = 1$ when $x_i = 1$ and $x_j = 0$ when $x_i = 0$. Therefore, combining the best responses of both players results in the best-response cycle.

\paragraph{Step 1}
Without loss of generality, we focus on instances of $N = 2$. This can be seen as a special case where there are more than $2$ players, but all of them are fixed except for two.

Let $\bft{Q} = (q^0_1, q^0_2)$ such that $\frac{d^2 \mathcal{M}_1(\bft{Q})}{dQ_1^2} = \frac{d^2 F_1(Q_1)}{dQ_1^2} < 0$. Then, there exists a small neighborhood, defined by $\varepsilon_1 > 0$ such that  for every $\delta_1, \delta_1' \in (0, \varepsilon_1]$ such that $\delta_1' > \delta_1$ it holds that
\begin{align} \label{eq second self deriv relation contradict}
\frac{d F_i(q^0_1 + \delta_1)}{d \delta_1} > \frac{d F_i(q^0_1 + \delta_1')}{d \delta_1'}     
\end{align}

We create an instance where we focus on the range $[q^0_1, q^0_1 + \varepsilon_1]$.
Let $\delta_1, \gamma \in (0, \frac{\varepsilon_1}{2}]$, such that $\gamma > \delta_1$. Furthermore, let $\delta_2 \in (0, 1]$ such that $q^0_2 + \delta_2 \leq 1$. In addition, Let $\lambda, \tau_1, \tau_2 \in (0, 1]$. We define the qualities by
\begin{align*}
Q_1(x_1, x_2) = q^0_1 + \delta_1 x_1 + \gamma x_2 \qquad Q_2(x_1, x_2) = q^0_2 + \delta_2 x_2(1 + \lambda x_1),
\end{align*}
and the costs by $c_i(x_i) = \tau_i x_i$.

\paragraph{Step 2}
Let $B_i(x_j)$ be the best response of player $i$ to the action $x_j$ of player $j$. Then our goal now is to show that $B_1(0) = 1$ and $B_1(1) = 0$. For that, we consider derivative of the utility of player $1$:
\begin{align*}
\frac{du_1}{dx_1} = \frac{d \mathcal{M}_1}{dx_1} - \frac{dc_1}{dx_1} = \frac{d \mathcal{F}_1}{dQ_1} \delta_1 - \tau_1.
\end{align*}
Therefore, we get that
\begin{align*}
\begin{cases}
\frac{du_1}{dx_1} = \frac{d \mathcal{F}_1(q^0_1 + \delta_1 x_1)}{dQ_1}\delta_1 - \tau_1 & \mbox{$x_2 = 0$} \\
\frac{du_1}{dx_1} = \frac{d F_1(q^0_1 + \delta_1 x_1 + \gamma))}{dQ_1}\delta_1 - \tau_1 & \mbox{$x_2 = 1$}
\end{cases}
\end{align*}

From Inequality~\eqref{eq second self deriv relation contradict} it holds that
\[
\frac{d \mathcal{F}_1(q^0_1 + \delta_1 x_1)}{dQ_1} > \frac{d F_1(q^0_1 + \delta_1 x_1 + \gamma))}{dQ_1}.
\]

Since $\gamma \geq \delta_1$, for every $x_1$, it holds that
\[
q_1^0 + \delta_1 x_1 \leq \max_{x_1} \{q_1^0 + \delta_1 x_1 \} = q_1^0 + \delta_1 < q_1^0 + \gamma = \min_{x_i} \{ q_1^0 + \delta_1 x_1 + \gamma \} \leq q_1^0 + \delta_1 x_1 + \gamma.
\]

Therefore, from Axiom~\ref{axiom: self promote} we get that
\[
\min_{x_i} \frac{d \mathcal{F}_1(q^0_1 + \delta_1 x_1)}{dQ_1} > \max_{x_1} \frac{d F_1(q^0_1 + \delta_1 x_1 + \gamma))}{dQ_1}.
\]

Specifically, there exists $\gamma > 0$ such that
\[
\min_{x_i} \frac{d \mathcal{F}_1(q^0_1 + \delta_1 x_1)}{dQ_1} - \max_{x_1} \frac{d F_1(q^0_1 + \delta_1 x_1 + \gamma))}{dQ_1} = \eta.
\]

Let $\tau_1 = \max_{x_1} \frac{d F_1(q^0_1 + \delta_1 x_1 + \gamma))}{dQ_1} \delta_1 + \frac{\eta}{2} \delta_1 $ and notice that
\begin{align*}
\frac{du_1(x_1, 0)}{dx_1} &= \frac{d \mathcal{F}_1(q^0_1 + \delta_1 x_1)}{dQ_1}\delta_1 - \tau_1 \\
&\geq \min_{x_i} \frac{d \mathcal{F}_1(q^0_1 + \delta_1 x_1)}{dQ_1} \delta_1 - \tau_1 \\
&= \max_{x_1} \frac{d F_1(q^0_1 + \delta_1 x_1 + \gamma))}{dQ_1} \delta_1 + \eta \delta - \tau_1 = \frac{\eta}{2} \delta_1 > 0.
\end{align*}

and

\begin{align*}
\frac{du_1(x_1, 1)}{dx_1} &= \frac{d F_1(q^0_1 + \delta_1 x_1 + \gamma))}{dQ_1}\delta_1 - \tau_1 \\
&\leq \max_{x_1} \frac{d F_1(q^0_1 + \delta_1 x_1 + \gamma))}{dQ_1}\delta_1 - \tau_1 = -\frac{\eta}{2} \delta_1 < 0.
\end{align*}
Since those inequalities holds for every $x_1$, we get that $B_1(0) = 1$ and $B_1(1) = 0$.

\paragraph{Step 3} we now repeat the same type of arguments for player 2. Starting with the utility
\begin{align*}
\frac{du_2}{dx_2} = \frac{d\mathcal{M}_2}{dx_2} - \frac{dc_2}{dx_2} = \frac{dF_2(Q_2(x_1, x_2))}{dx_2} - \tau_2 = \frac{dF_2}{dQ_2} \delta_2\left(1 + \lambda x_1 \right) - \tau_2.
\end{align*}
Therefore,
\begin{align*}
\begin{cases}
\frac{du_2}{dx_2} = \frac{d F_2 (q^0_2 + x_2 \delta_2)}{dQ_2} \delta_2  - \tau_2 & \mbox{$x_1 = 0$} \\
\frac{du_2}{dx_2} = \frac{d F_2 (q^0_2 + x_2 \delta_2 (1 + \lambda))}{dQ_2} \delta_2 (1 + \lambda)  - \tau_2 & \mbox{$x_1 = 1$}
\end{cases}
\end{align*}

Since $\frac{d\mathcal{M}}{dQ_2}$ is continuous, for every $\eta_2 > 0$, there exists $\delta_2$ small enough such that for every $x_2, x_2'$ it holds that
\begin{align*}
\left| \frac{d F(q^0_2 + x_2 \delta_2)}{dQ_2} - \frac{dF_2 (q^0_2 + x_2' \delta_2 (1 + \lambda))}{dQ_2} \right| \leq \eta_2.
\end{align*}

By choosing $\eta_2 = \min_{x_2} \frac{d F_2}{dQ_2} \delta_2 (1 + 0.5\lambda) = \frac{dF_2(q_2^0)}{dQ_2} \delta_2 (1 + 0.5 \lambda)$, and $\delta_2$ such that $\eta_2 < 0.5 \min_{x_2} \frac{d \mathcal{F}_2}{dQ_2} \lambda = 0.5 \frac{d \mathcal{F}_2(q_2^0)}{dQ_2} \lambda$ we get that for $\tau_2 = \min_{x_1, x_2} \frac{d \mathcal{M}_2}{dQ_2} \delta_2 (1 + 0.5\lambda)$ it holds that

\begin{align*}
\frac{du_2(0, x_2)}{dx_2} &= \frac{d F_2 (q^0_2 + x_2 \delta_2)}{dQ_2} \delta_2  - \tau_2 \\
&\leq \min_{x_2} \frac{d F_2}{dQ_2} \delta_2 + \eta \delta_2 - \min_{x_1, x_2} \frac{d \mathcal{M}_2}{dQ_2} \delta_2 (1 + 0.5\lambda) \\
&= -0.5 \min_{x_1, x_2} \frac{d \mathcal{M}_2}{dQ_2} \delta_2 \lambda + \eta \delta_2 < 0.
\end{align*}

and 

\begin{align*}
\frac{du_2(1, x_2)}{dx_2} &= \frac{d\mathcal{M}_2 (q^0_1 + \delta_1, q^0_2 + x_2 \delta_2 (1 + \lambda))}{dQ_2} \delta_2 (1 + \lambda)  - \tau_2 \\
&\geq \min_{x_1, x_2} \frac{d \mathcal{M}_2}{dQ_2} \delta_2 (1 + \lambda) - \tau_2 \\
&= \min_{x_1, x_2} \frac{d \mathcal{M}_2}{dQ_2} \delta_2 (1 + \lambda)  - \min_{x_1, x_2} \frac{d \mathcal{M}_2}{dQ_2} \delta_2 (1 + 0.5\lambda) \\
&= 0.5 \min_{x_1, x_2} \frac{d \mathcal{M}_2}{dQ_2} \delta_2 \lambda > 0.
\end{align*}

Hence, we got that $0 = B_2(0)$ and $1 = B_2(1)$, which means that there is a best response cycle where the players move between the profiles: 
\begin{align*}
(0, 0) \rightarrow (1, 0) \rightarrow (1,1) \rightarrow (0, 1) \rightarrow (0, 0),
\end{align*}
which contradicts Axiom~\ref{axiom: no cycles}. This concludes the proof of \Cref{lemma mechanism is convex}.
\end{proofof}

\section{Proofs Omitted From Section~\ref{sec: mechanism suggestion}}

\begin{definition}[Component-Wise Ordering]
\label{def:componentwise}
For two strategy profiles $\mathbf{x}, \mathbf{x}' \in [0,1]^N$, we write $\mathbf{x} \ge \mathbf{x}'$ if $x_i \ge x_i'$ for every $i \in \mathcal{N}$.
\end{definition}

\begin{lemma}[Supermodularity of the Induced Game]
\label{lemma:supermodularity}
For every mechanism $\mathcal{M} \in \mathcal{F}^{PRA}$, every induced game $\mathcal{G} \in \mathcal{G}(\mathcal{M})$ is a supermodular game.
\end{lemma}

\begin{proof}
The strategy space for each player is the compact interval $[0,1]$, which forms a complete lattice under the standard ordering (Definition~\ref{def:componentwise}). It remains to show that the utility function $U_i(\mathbf{x})$ exhibits \emph{increasing differences} in $(x_i, \mathbf{x}_{-i})$.
Under the Provisional Allocation mechanism, the utility is given by:
\[
U_i(\mathbf{x}) = p_i Q_i(\mathbf{x}) - c_i(x_i).
\]
Since the cost function $c_i(x_i)$ depends only on $x_i$, the cross-partial derivative of utility with respect to $x_i$ and any competitor's effort $x_j$ is determined solely by the quality term:
\[
\frac{\partial^2 U_i}{\partial x_i \partial x_j} = p_i \frac{\partial^2 Q_i}{\partial x_i \partial x_j}.
\]
By definition of the mechanism, $p_i \ge 0$. By Assumption~\ref{increment plus} (Effort Complementarities), $\frac{\partial^2 Q_i}{\partial x_i \partial x_j} \ge 0$. Consequently, $\frac{\partial^2 U_i}{\partial x_i \partial x_j} \ge 0$ for all $j \neq i$. This sufficient condition establishes that the game is supermodular.
\end{proof}

\begin{proofof}{thm:stability_selection}
By Lemma~\ref{lemma:supermodularity}, every induced game $\mathcal{G} \in \mathcal{G}(\mathcal{M})$ is supermodular. We now derive each property from this structure.

\paragraph{Part~\ref{item:stability} (Stability) and Part~\ref{item:reachable} (Reachability)}
These properties follow directly from the fundamental theorems of supermodular games:
\begin{itemize}
    \item \textbf{Stability:} By Topkis's Theorem (\cite{topkis1998supermodularity}, Theorem 4.2.1), every supermodular game possesses a non-empty set of pure Nash equilibria, which contains a greatest equilibrium $\overline{\mathbf{x}}$ (and a least equilibrium). The greatest equilibrium satisfies $\overline{\mathbf{x}} \ge \mathbf{x}$ (Definition~\ref{def:componentwise}) for any other PNE $\mathbf{x}$.
    \item \textbf{Reachability:} As shown by Vives \cite{vives2018supermodularity}, iterating best-responses starting from the maximal profile $\mathbf{1}=(1,\dots,1)$ converges monotonically to the greatest equilibrium $\overline{\mathbf{x}}$.
\end{itemize}

\paragraph{Part~\ref{item:sw} (Welfare Dominance)}
We must show $SW(\overline{\mathbf{x}}) \ge SW(\mathbf{x})$.
Recall that $SW(\mathbf{x}) = \sum_{j \in \mathcal{N}} Q_j(\mathbf{x})$. Differentiating with respect to any effort $x_k$:
\[
\frac{\partial SW}{\partial x_k} = \sum_{j \in \mathcal{N}} \frac{\partial Q_j}{\partial x_k}.
\]
By the assumption of non-negative spillovers, $\frac{\partial Q_j}{\partial x_k} \ge 0$ for all $j,k$. Therefore, the social welfare function is monotonically non-decreasing in the strategy profile. Since $\overline{\mathbf{x}} \ge \mathbf{x}$ component-wise, it follows immediately that $SW(\overline{\mathbf{x}}) \ge SW(\mathbf{x})$.

\paragraph{Part~\ref{item:utility} (Utility Dominance)}
Let $\overline{\mathbf{x}}$ be the greatest PNE and $\mathbf{x}$ be any other PNE. We show that $U_i(\overline{\mathbf{x}}) \ge U_i(\mathbf{x})$ for all $i$.

Since $\overline{\mathbf{x}}$ is a Nash equilibrium, player $i$ plays a best response to $\overline{\mathbf{x}}_{-i}$. Therefore, deviating to any other strategy $x_i$ (where $x_i$ is player~$i$'s strategy in the alternative equilibrium $\mathbf{x}$) cannot be strictly profitable:
\begin{equation} \label{eq:nash_condition}
    U_i(\overline{x}_i, \overline{\mathbf{x}}_{-i}) \ge U_i(x_i, \overline{\mathbf{x}}_{-i}).
\end{equation}
Next, consider the term $U_i(x_i, \overline{\mathbf{x}}_{-i})$. Since $\overline{\mathbf{x}} \ge \mathbf{x}$, we have $\overline{\mathbf{x}}_{-i} \ge \mathbf{x}_{-i}$. Due to non-negative spillovers, $Q_i$ is non-decreasing in $\mathbf{x}_{-i}$, which implies:
\[
Q_i(x_i, \overline{\mathbf{x}}_{-i}) \ge Q_i(x_i, \mathbf{x}_{-i}).
\]
Multiplying by $p_i \ge 0$ and subtracting the cost $c_i(x_i)$ (which is identical on both sides) yields:
\begin{equation} \label{eq:monotonicity_condition}
    U_i(x_i, \overline{\mathbf{x}}_{-i}) \ge U_i(x_i, \mathbf{x}_{-i}) = U_i(\mathbf{x}).
\end{equation}
Combining (\ref{eq:nash_condition}) and (\ref{eq:monotonicity_condition}), we obtain:
\[
U_i(\overline{\mathbf{x}}) \ge U_i(x_i, \overline{\mathbf{x}}_{-i}) \ge U_i(\mathbf{x}).
\]
Thus, the greatest equilibrium Pareto-dominates any other equilibrium.
\end{proofof}

\begin{proofof}{max q np-hard}
We reduce from the Maximum Clique problem. Given a graph $G=(V, E)$, our goal is to find a maximum cardinality subset $V' \subseteq V$ such that $(v, v') \in E$ for every $v, v' \in V'$.

\paragraph{Construction}
From $G=(V, E)$, we construct an instance of Problem~\eqref{eq:design-problem} as follows. Let $N = |V|$ be the number of players. Using the graph-based spillover model (Example~\ref{graph example}), define the quality function for each player $i$ as:
\[
Q_i(\mathbf{x}) = \frac{1}{N} \left( x_i + \sum_{\substack{j \neq i \\ (i, j) \in E}} x_i x_j \right),
\]
and the cost function as $c_i(x_i) = x_i / N$. Under the Provisional Allocation mechanism with parameter~$\mathbf{p}$, the utility of player $i$ is:
\[
U_i(\mathbf{x}) = p_i \cdot Q_i(\mathbf{x}) - c_i(x_i) = \frac{x_i}{N}\left( p_i \left( 1 + \sum_{\substack{j \neq i \\ (i, j) \in E}} x_j \right) - 1 \right).
\]

\paragraph{Binary best responses}
Since $U_i$ is linear in $x_i$, each player's best response is binary: $x_i \in \{0, 1\}$. Player $i$ chooses $x_i = 1$ if and only if the coefficient of $x_i$ in the utility expression is non-negative, i.e.,
\[
p_i \left( 1 + \sum_{\substack{j \neq i \\ (i, j) \in E}} x_j \right) \geq 1.
\]
Equivalently, if $n$ neighbors of player $i$ play $x_j = 1$, then player $i$ plays $x_i = 1$ only if $p_i \geq 1/(n+1)$.

\paragraph{Structure of optimal solutions}
The following lemma characterizes the equilibrium structure:

\begin{lemma} \label{sol is clique}
Let $(p_i)_i$ be a feasible allocation with $p_i \geq 0$ and $\sum_{i} p_i \leq 1$. Then there exists a subset $\mathcal{N}' \subseteq \mathcal{N}$ such that:
\begin{enumerate}
    \item \label{item:binary} The dominant equilibrium satisfies $x^\star_i = 1$ if $i \in \mathcal{N}'$ and $x^\star_i = 0$ otherwise.
    \item \label{item:uniform} If $|\mathcal{N}'| > 0$, then $p_i = 1/|\mathcal{N}'|$ for all $i \in \mathcal{N}'$ and $p_i = 0$ otherwise.
    \item \label{item:clique} The set $\mathcal{N}'$ forms a clique: for every $i, j \in \mathcal{N}'$ with $i \neq j$, we have $(i, j) \in E$.
\end{enumerate}
\end{lemma}

The proof of Lemma~\ref{sol is clique} appears below. By this lemma, any feasible allocation induces either no activity (all $x_i = 0$) or a clique among active players. Since $Q_i(\mathbf{x}^\star) = 0$ for inactive players and $Q_i(\mathbf{x}^\star) = |\mathcal{N}'|/N$ for each $i \in \mathcal{N}'$, the social welfare is:
\[
SW = \sum_{i \in \mathcal{N}'} Q_i(\mathbf{x}^\star) = |\mathcal{N}'| \cdot \frac{|\mathcal{N}'|}{N} = \frac{|\mathcal{N}'|^2}{N}.
\]
This is strictly increasing in $|\mathcal{N}'|$. Therefore, the optimal mechanism corresponds to the maximum clique in $G$, and solving Problem~\eqref{eq:design-problem} solves Maximum Clique. Since Maximum Clique is NP-hard, so is Problem~\eqref{eq:design-problem}.
\end{proofof}

\begin{proofof}{sol is clique}
Part~\ref{item:binary} holds trivially since each player's utility is linear in their own action, forcing best responses to be binary.

\textbf{Parts~\ref{item:uniform} and~\ref{item:clique}.} Let $\mathbf{x}^\star$ denote the equilibrium profile induced by the allocation $(p_i)_i$, and let $\mathcal{N}' = \{i \in \mathcal{N} : x^\star_i = 1\}$. Assume $|\mathcal{N}'| > 0$.

Since $Q_i(\mathbf{x}^\star) = 0$ whenever $x^\star_i = 0$, only players in $\mathcal{N}'$ have positive quality. Next, let $i^\star \in \arg\max_{i \in \mathcal{N}'} Q_i(\mathbf{x}^\star)$. Define $\mathcal{N}_{i^\star}$ such that
\[
\mathcal{N}_{i^\star} = \{ j \in \mathcal{N} : j \neq i^\star,\, (i^\star, j) \in E,\, x^\star_j = 1 \},
\]
i.e., the set of neighbors of $i^\star$ in the graph $G$ who play $x^\star_j = 1$. By the quality function structure:
\[
Q_{i^\star}(\mathbf{x}^\star) = \frac{1}{N}\left(1 + |\mathcal{N}_{i^\star}|\right).
\]

For $i^\star$ to choose $x^\star_{i^\star} = 1$, the best-response condition requires:
\[
p_{i^\star} \geq \frac{1}{1 + |\mathcal{N}_{i^\star}|}.
\]

For any $j \in \mathcal{N}_{i^\star}$, since $i^\star$ maximizes quality among active players, we have $Q_j(\mathbf{x}^\star) \leq Q_{i^\star}(\mathbf{x}^\star)$. The best-response condition for $j$ requires:
\[
p_j \geq \frac{1}{N \cdot Q_j(\mathbf{x}^\star)} \geq \frac{1}{N \cdot Q_{i^\star}(\mathbf{x}^\star)} = \frac{1}{1 + |\mathcal{N}_{i^\star}|}.
\]

Summing over $\{i^\star\} \cup \mathcal{N}_{i^\star}$:
\[
p_{i^\star} + \sum_{j \in \mathcal{N}_{i^\star}} p_j \geq \frac{1 + |\mathcal{N}_{i^\star}|}{1 + |\mathcal{N}_{i^\star}|} = 1.
\]

Since $\sum_i p_i \leq 1$, this inequality is tight and $p_k = 0$ for all $k \notin \{i^\star\} \cup \mathcal{N}_{i^\star}$. Any player with $p_k = 0$ cannot satisfy the best-response condition for $x_k = 1$, hence $x^\star_k = 0$. Thus $\mathcal{N}' = \{i^\star\} \cup \mathcal{N}_{i^\star}$, and $p_i = 1/|\mathcal{N}'|$ for all $i \in \mathcal{N}'$, establishing Part~\ref{item:uniform}.

For Part~\ref{item:clique}, the tight inequality chain implies $Q_j(\mathbf{x}^\star) = Q_{i^\star}(\mathbf{x}^\star)$ for every $j \in \mathcal{N}'$. Thus each player in $\mathcal{N}'$ has exactly $|\mathcal{N}'| - 1$ neighbors within $\mathcal{N}'$, meaning every pair of players in $\mathcal{N}'$ shares an edge.
\end{proofof}

\begin{theorem} \label{np-hard when simple}
Problem~\eqref{eq:design-problem} is NP-hard even under the following simplifying assumptions:
\begin{enumerate}
    \item There are no spillovers.
    \item Both the quality function $Q_i(x_i)$ and the cost function $c_i(x_i)$ are linear in $x_i$.
\end{enumerate}
\end{theorem}

\begin{proofof}{np-hard when simple}

We prove NP-hardness by a polynomial-time reduction from the classical 0--1 knapsack optimization problem.

\textbf{Instance of 0--1 knapsack.}  
Let an instance of 0--1 knapsack be given by:
\[
\text{Items } i=1,\dots,N,\qquad \{v_i\}_{i=1}^N\subset\mathbb{R}_{\ge 0},\qquad \{w_i\}_{i=1}^N\subset\mathbb{R}_{>0},\qquad\text{capacity }W>0,
\]
and the knapsack optimization problem is
\begin{equation}\label{eq:knapsack}
\max_{x\in\{0,1\}^N}\ \sum_{i=1}^N v_i x_i
\quad\text{s.t.}\quad \sum_{i=1}^N w_i x_i \le W.
\end{equation}
Without loss of generality, every item satisfies $w_i \le W$ (an item with $w_i > W$ cannot be included in any feasible solution and may be deleted) and $v_i, w_i > 0$. Define normalized weights \(\tilde w_i := \dfrac{w_i}{W} \in (0,1]\) and normalized values \(\tilde v_i := \dfrac{v_i}{v_{\max}} \in (0,1]\), where \(v_{\max} := \max_i v_i\). Then the constraint in \eqref{eq:knapsack} is equivalent to \(\sum_{i=1}^N \tilde w_i x_i \le 1\).

\textbf{Constructed instance of Problem~\eqref{eq:design-problem}.}  
From the knapsack instance we construct an instance of Problem~\eqref{eq:design-problem} with \(N\) players as follows:
\[
Q_i(x_i) := \tilde v_i x_i,\qquad c_i(x_i) := \tilde v_i \tilde w_i x_i \qquad(\text{both linear in }x_i).
\]
The mechanism assigns player \(i\) allocation share \(p_i\), so player \(i\)'s utility is
\[
U_i(x_i;p_i) = p_i Q_i(x_i) - c_i(x_i) = \tilde v_i x_i (p_i - \tilde w_i).
\]
Players simultaneously choose \(x_i\in\{0,1\}\) to maximize their utility given $\bft{p}$. 

\textbf{Best-response characterization.}  
For each fixed $\bft{p}$, the effort of player \(i\) in the greatest equilibrium is
\begin{equation}\label{eq:br}
\overline{x}_i(\bft{p}) \;=\; 
\begin{cases}
1, & \text{if } p_i \ge \tilde w_i,\\[6pt]
0, & \text{if } p_i < \tilde w_i,
\end{cases}
\end{equation}
where at \(p_i = \tilde w_i\) player \(i\) is indifferent, and the tie is resolved toward \(\overline{x}_i = 1\) by the greatest-equilibrium selection.

\begin{lemma}\label{lemma:threshold}
Let $\bft{p}$ be any feasible choice for the designer (i.e., \(p_i\ge 0\) and \(\sum_i p_i\le 1\)). Define the index set
\[
S := \{i\in\{1,\dots,N\}\;:\; p_i \ge \tilde w_i\}.
\]
Define \(\bft{p}'\in\mathbb{R}_{\ge 0}^N\) by
\[
p'_i := 
\begin{cases}
\tilde w_i, & i\in S,\\[4pt]
0, & i\notin S.
\end{cases}
\]
Then $\bft{p}'$ is feasible (\(\sum_i p'_i \le 1\)) and induces the same greatest equilibrium as $\bft{p}$, i.e., \(\overline{\bft{x}}(\bft{p}') = \overline{\bft{x}}(\bft{p})\). In particular, for any feasible $\bft{p}$ there exists a feasible $\bft{p}'$ of this threshold form that yields the same objective value $\sum_i Q_i(\overline{\bft{x}}(\cdot))$.
\end{lemma}

\textbf{Equivalence to knapsack.}  
Given Lemma \ref{lemma:threshold}, any feasible mechanism \(\bft{p}\) can be replaced by a $\bft{p}'$ that places mass exactly \(\tilde w_i\) on every player \(i\) who is activated, and zero otherwise. Consequently, feasible allocations \(\bft{p}\) correspond one-to-one to subsets \(S\subseteq\{1,\dots,N\}\) satisfying \(\sum_{i\in S}\tilde w_i \le 1\); the induced objective value is
\[
\sum_{i=1}^N Q_i\left(\overline{x}_i(\bft{p})\right) \;=\; \sum_{i\in S} \tilde v_i \;=\; \frac{1}{v_{\max}}\sum_{i\in S} v_i.
\]
Since \(1/v_{\max}\) is a fixed positive constant, the designer's optimization problem on the constructed instance is equivalent to
\[
\max_{S\subseteq\{1,\dots,N\}} \ \sum_{i\in S} v_i
\quad\text{s.t.}\quad \sum_{i\in S} \tilde w_i \le 1,
\]
which is exactly the 0--1 knapsack instance \eqref{eq:knapsack} (after normalization of weights).

\textbf{Conclusion.}  
Since 0--1 knapsack is NP-hard, the designer's optimization problem is NP-hard under the stated linearity and no-externality assumptions. This completes the proof.

\end{proofof}

\textbf{Proof of Lemma \ref{lemma:threshold}.}  
Let $S$ and $\bft{p}'$ be as defined above. First,
\[
\sum_{i=1}^N p'_i \;=\; \sum_{i\in S} \tilde w_i \;\le\; \sum_{i\in S} p_i \;\le\; \sum_{i=1}^N p_i \;\le\; 1,
\]
where the first inequality uses \(\tilde w_i \le p_i\) for \(i\in S\). Hence $\bft{p}'$ is feasible.  
Second, for each \(i\) we compare \(p'_i\) with \(\tilde w_i\):
\[
p'_i \ge \tilde w_i \iff i\in S \iff p_i \ge \tilde w_i.
\]
Therefore by the best-response rule \eqref{eq:br} we have \(\overline{x}_i(\bft{p}') = \overline{x}_i(\bft{p})\) for every \(i\). Consequently the induced quality vector is identical, and so the social welfare value is the same. This proves the lemma.

\begin{proofof}{no approximation}
We proceed by contradiction. Suppose there exists a polynomial-time algorithm
that approximates the $SW$ problem within a factor of $N^{1-\veps}$ for some
$\veps>0$. Apply it to the instance constructed in the proof of \Cref{max q np-hard}. Its output allocation induces a clique of size $k$ with
$k^2/N \ge \omega(G)^2/(N\cdot N^{1-\veps})$, i.e.,
$k \ge \omega(G)/N^{1-\veps'}$ for $\veps' = \frac{1+\veps}{2} > 0$,
where $\omega(G)$ is the maximum clique size.
This is a polynomial-time $N^{1-\veps'}$-approximation for Maximum Clique,
contradicting \cite[Theorem 5.2]{Hastad1999}, which states that unless
$NP=ZPP$, no such approximation exists for any $\veps'>0$.
\end{proofof}


\section{Proofs Omitted From \Cref{sec: bounded}} \label{bounded appnx}

\paragraph{Bounds in the definition versus inputs to the analysis.}
We first record a monotonicity that lets the analysis operate with a single pair of constants. If
$Q_i$ has $(\beta,\eta)$-bounded spillovers (Definition~\ref{def:beta_bounded}), then it has
$(\beta',\eta')$-bounded spillovers for every $\beta'\ge\beta$ and $\eta'\ge\eta$, since both defining
inequalities only weaken as the constants grow. Hence any valid upper bounds may be used in the
analysis; we denote them by $\beta,\eta$ and assume, without loss of generality, that they are the
constants of Definition~\ref{def:beta_bounded}.

\begin{proofof}{thm:approx_bounded}

Let $\mathbf{p}^\star\in\arg\max_{\mathbf{p}\in\mathcal{B}(\veps)^N,\,\sum_i p_i\le 1}SW(\overline{\mathbf{x}}(\mathbf{p}))$
be an optimal allocation and $\mathbf{x}^\star:=\overline{\mathbf{x}}(\mathbf{p}^\star)$ the equilibrium it
induces; let $\hat{\mathbf{p}}$ be the output of Algorithm~\ref{alg:NSR_solver} and
$\hat{\mathbf{x}}:=\overline{\mathbf{x}}(\hat{\mathbf{p}})$. Note that $SW(\mathbf{x}^\star)=OPT$. The runtime follows from the following lemma:

\begin{lemma} \label{lemma:mckp}
The optimal solution $\hat{\mathbf{p}}$ to the NSR problem can be computed in
$O(N/\veps^2)$ time.
\end{lemma}%

Proofs of all relevant lemmas are deferred to immediately after the present proof.

\paragraph{Setup and notation.}
\label{par:br_setup}
Recall that $U_i(\mathbf{x};\mathbf{p}) := p_i\,Q_i(\mathbf{x}) - c_i(x_i)$ is player~$i$'s utility under
allocation $\mathbf{p}$ at effort profile $\mathbf{x}$, and that $\overline{\mathbf{x}}(\mathbf{p})$ denotes the
greatest equilibrium induced by $\mathbf{p}$. For a fixed profile $\mathbf{x}_{-i}\in[0,1]^{N-1}$ of the
others, define player~$i$'s \emph{best-response effort} to an allocation share $p\ge 0$ as
\[
\mathrm{BR}_i(p \,;\, \mathbf{x}_{-i}) \;:=\; \max\,\arg\max_{z\in[0,1]} U_i\left((z,\mathbf{x}_{-i});\,p\right),
\]
the largest maximizer (well defined, as the objective is continuous on $[0,1]$). We further define,
for every share $p \ge 0$, the \emph{no-spillover best response}
\[
y_i(p) \;:=\; \mathrm{BR}_i(p \,;\, \mathbf{0}_{-i}) \;=\; \max\,\arg\max_{z\in[0,1]} \left\{ p\, Q_i(z,\mathbf{0}_{-i}) - c_i(z) \right\},
\]
the effort player $i$ chooses under share $p$ when all others are inactive. In particular, for any
allocation $\mathbf{p}$, the auxiliary variables $(y_i)_i$ appearing in the NSR problem of
Algorithm~\ref{alg:NSR_solver} coincide with $\left(y_i(p_i)\right)_i$. Under the largest-maximizer
tie-breaking, the NSR objective value of $\mathbf{p}$ equals $\sum_i Q_i\left(y_i(p_i),\mathbf{0}_{-i}\right)$.
 
We also use the following notation throughout: $Q_{\max}:=\max_i\max_{\mathbf{x}} Q_i(\mathbf{x})\le 1$.
 
\medskip
Our argument relies on the following two lemmas. 
\begin{lemma}
\label{lem:price_comp}
Fix a player $i$, a share $p \ge 0$, and an arbitrary profile $\mathbf{x}_{-i}\in[0,1]^{N-1}$, and let
$x_i^{\dagger}:=\mathrm{BR}_i(p \,;\, \mathbf{x}_{-i})$ be player~$i$'s best response to share $p$ when the others
play $\mathbf{x}_{-i}$. Then,
\[
y_i\left((1+\eta)\,p\right)\;\ge\;x_i^{\dagger},
\]
where $y_i(\cdot) = \mathrm{BR}_i(\cdot \,;\, \mathbf{0}_{-i})$ is the no-spillover best response defined above.
 
That is, raising the share to $(1+\eta)p$ and removing all spillovers,
induces an effort at least as large as the original best response.
\end{lemma}%
 
\begin{lemma}
\label{lem:knapsack_abstract}
Consider $N$ items, where each item $i\in\{1,\dots,N\}$ has a weight $w_i\ge 0$ and a value
$v_i\ge 0$. Suppose the total weight satisfies $\sum_i w_i\le W$ for some $W\ge 1$, and let
$v_{\max}:=\max_i v_i$ denote the maximal value. Then there is a set of items
$S\subseteq\{1,\dots,N\}$ of total weight $\sum_{i\in S}w_i\le 1$ and total value
\[
\sum_{i\in S}v_i\;\ge\; \left(\frac1W\sum_i v_i\right)-\;v_{\max}.
\]
\end{lemma}%
 
We prove the welfare guarantee in four steps.
 
\paragraph{Step 1 (an inflated, grid-aligned allocation).}
We define the \emph{inflated} allocation $\bar{\mathbf{p}}=(\bar p_1,\dots, \bar p_N)$ such that $\bar p_i:=\veps\left\lceil (1+\eta)\,p_i^\star/\veps\right\rceil$, namely, the smallest multiple of
$\veps$ that is at least $(1+\eta)p_i^\star$.
We call it inflated since each entry inflates the corresponding entry of the optimal allocation; thus, it need not be a valid allocation, as $\sum_i\bar p_i$ may exceed 1. Nevertheless, it will be useful for our analysis -- later we will 'extract' a valid allocation out of it. We align $\bar{\mathbf{p}}$ to multiples of $\veps$ so that any allocation assembled from its entries, as constructed in Step~3 below, is a valid element of $\mathcal{B}(\veps)^N$.
 
Since $\lceil{t}\rceil \leq t+1$ for every $t \in \mathbb R$, we have $\bar p_i\le(1+\eta)p_i^\star+\veps$.
Since $\sum_i p_i^\star\le 1$, we have
\begin{equation}
\label{eq:weight}
\sum_i\bar p_i\;\le\;(1+\eta)\sum_i p_i^\star+N\veps\;\le\;1+\eta+N\veps.
\end{equation}
For ease of presentation, we denote $W:= 1+\eta+N\veps$.
 
\paragraph{Step 2 (the optimal efforts are attainable without spillovers under $\bar{\mathbf{p}}$).}
Fix an arbitrary player $i \in [N]$. In the equilibrium $\mathbf{x}^\star$, the effort $x_i^\star$ is player~$i$'s best response to
$p_i^\star$ given the others' efforts $\mathbf{x}^\star_{-i}$, i.e., \
$x_i^\star=\mathrm{BR}_i(p_i^\star \,;\, \mathbf{x}^\star_{-i})$. Lemma~\ref{lem:price_comp} with $p=p_i^\star$
and $\mathbf{x}_{-i}=\mathbf{x}^\star_{-i}$ gives $y_i\left((1+\eta)p_i^\star\right)\ge x_i^\star$. The no-spillover
best response $y_i$ is non-decreasing in the allocation share, $p \,$. Since $\bar p_i\ge(1+\eta)p_i^\star$,
\[
y_i(\bar p_i)\;\ge\;y_i\left((1+\eta)p_i^\star\right)\;\ge\;x_i^\star.
\]
Writing $v_i:=Q_i\left(y_i(\bar p_i),\mathbf{0}_{-i}\right)$, monotonicity of $Q_i$ in own effort yields
\begin{equation}
\label{eq:vdom}
v_i\;\ge\;Q_i(x_i^\star,\mathbf{0}_{-i})\qquad\text{for every player }i.
\end{equation}
 
\paragraph{Step 3 (compressing $\bar{\mathbf{p}}$ back to a valid allocation).}
Apply Lemma~\ref{lem:knapsack_abstract} with weights $w_i:=\bar p_i$ and values $v_i$: by
\eqref{eq:weight} the total weight is at most $W$, and $v_{\max}=\max_i v_i\le Q_{\max}$. This yields a
set $S$ with $\sum_{i\in S}\bar p_i\le 1$ and $\sum_{i\in S}v_i\ge\frac1W\sum_i v_i-Q_{\max}$. Let
$\mathbf{p}'$ be the allocation with $p_i'=\bar p_i$ for $i\in S$ and $p_i'=0$ otherwise: each $p_i'$ is a
multiple of $\veps$ and $\sum_i p_i'=\sum_{i\in S}\bar p_i\le 1$, so $\mathbf{p}'\in\mathcal{B}(\veps)^N$ is a
valid input to the NSR problem. Its NSR objective value satisfies
\[
\sum_i Q_i\left(y_i(p_i'),\mathbf{0}_{-i}\right)
=\sum_{i\in S}v_i+\sum_{i\notin S}Q_i\left(y_i(0),\mathbf{0}_{-i}\right)
\;\ge\;\sum_{i\in S}v_i.
\]
As $\hat{\mathbf{p}}$ maximizes the NSR objective over $\mathcal{B}(\veps)^N$, its objective value is at least
that of $\mathbf{p}'$; combining with \eqref{eq:vdom},
\begin{equation}
\label{eq:nsr_lb}
\sum_i Q_i\left(y_i(\hat p_i),\mathbf{0}_{-i}\right)
\;\ge\;\sum_{i\in S}v_i
\;\ge\;\frac1W\sum_i v_i-Q_{\max}
\;\ge\;\frac{1}{1+\eta+N\veps}\sum_i Q_i(x_i^\star,\mathbf{0}_{-i})-Q_{\max}.
\end{equation}
 
\paragraph{Step 4 (deploying $\hat{\mathbf{p}}$ in the actual game, and the $\beta$-bound).}
Denote by $\mathbf{y}:=\left(y_1(\hat p_1),\dots,y_N(\hat p_N)\right)$ the profile of no-spillover best
responses under $\hat{\mathbf{p}}$. Each $y_i(\hat p_i)$ is player~$i$'s best response to $\hat p_i$ when the
others are inactive ($\mathbf{0}_{-i}$). In the actual game, since it is supermodular, if the other players
exert efforts $\mathbf{y}_{-i}\ge\mathbf{0}_{-i}$, player~$i$'s best response is non-decreasing in them; hence
$\mathrm{BR}_i(\hat p_i \,;\, \mathbf{y}_{-i}) \geq \mathrm{BR}_i(\hat p_i \,;\, \mathbf{0}_{-i}) \geq y_i(\hat p_i)$.
Thus iterating the monotone best-response map from $\mathbf{y}$ converges upward to an equilibrium
dominated by the greatest equilibrium $\hat{\mathbf{x}}$. Therefore $\hat{\mathbf{x}}\ge\mathbf{y}$. As every $Q_i$
is non-decreasing in all efforts and $\mathbf{y}_{-i}\ge\mathbf{0}_{-i}$,
\begin{equation}
\label{eq:deploy}
SW(\hat{\mathbf{x}})=\sum_i Q_i(\hat{\mathbf{x}})\;\ge\;\sum_i Q_i(\mathbf{y})\;\ge\;\sum_i Q_i\left(y_i(\hat p_i),\mathbf{0}_{-i}\right).
\end{equation}
Finally, the first condition of Definition~\ref{def:beta_bounded}, applied at $\mathbf{x}^\star$ termwise,
gives $Q_i(\mathbf{x}^\star)\le(1+\beta)\,Q_i(x_i^\star,\mathbf{0}_{-i})$, so
\begin{equation}
\label{eq:level}
\sum_i Q_i(x_i^\star,\mathbf{0}_{-i})\;\ge\;\frac{1}{1+\beta}\sum_i Q_i(\mathbf{x}^\star)=\frac{1}{1+\beta}\,SW(\mathbf{x}^\star).
\end{equation}
Chaining \eqref{eq:deploy}, \eqref{eq:nsr_lb}, and \eqref{eq:level}, and using $Q_{\max}\le 1$,
\[
SW(\hat{\mathbf{x}})
\;\ge\;\frac{1}{1+\eta+N\veps}\sum_i Q_i(x_i^\star,\mathbf{0}_{-i})-Q_{\max}
\;\ge\;\frac{1}{(1+\beta)(1+\eta+N\veps)}\,SW(\mathbf{x}^\star)-1,
\]

which is the guarantee of Theorem~\ref{thm:approx_bounded} with the term $\eta+N\veps$ in the minimum.

\paragraph{The alternative term $\lceil\eta\rceil$.}
$(\beta,\eta)$-bounded spillovers imply $(\beta,\lceil\eta\rceil)$-bounded spillovers, so Steps~1--4
apply verbatim with $\eta$ replaced by $\lceil\eta\rceil$. In Step~1, since $1+\lceil\eta\rceil$ is an
integer and each $p_i^\star$ is a multiple of $\veps$, the product $(1+\lceil\eta\rceil)\,p_i^\star$ is
itself a multiple of $\veps$. The rounding is therefore vacuous, $\bar p_i = (1+\lceil\eta\rceil)\,p_i^\star$,
and \eqref{eq:weight} improves to $\sum_i \bar p_i \le 1+\lceil\eta\rceil$. Carrying
$W = 1+\lceil\eta\rceil$ through Steps~2--4 unchanged yields
\[
SW(\hat{\mathbf{x}}) \;\ge\; \frac{1}{(1+\beta)\left(1+\lceil\eta\rceil\right)}\,SW(\mathbf{x}^\star) - 1.
\]
Combining the two bounds yields the guarantee of Theorem~\ref{thm:approx_bounded}.

\end{proofof}
 

\begin{proofof}{lemma:mckp} \label{proofof mckp}
We show that our allocation problem can be mapped exactly to an instance of the 0-1 Multiple-Choice Knapsack Problem (MCKP) \cite{Kellerer2004}.
In the 0-1 MCKP, we are given disjoint classes of items, and we must select exactly one item from each class to maximize total profit subject to a weight capacity constraint. We construct the mapping as follows:
\begin{itemize}
    \item \textbf{Classes:} The classes correspond to the players $i \in \mathcal{N}$ (total $N$ classes).
    \item \textbf{Items:} For each class $i$, the items are indexed by $k \in \{0, 1, \ldots, \lfloor 1/\veps \rfloor\}$, where item $k$ corresponds to the allocation  $p_i =k\veps$.
    \item \textbf{Weights:} The weight of item $k$ is $w_{ik} = k$.
    \item \textbf{Profits:} The profit of item $k$ in class $i$ is the quality induced by allocation $k\veps$ to player $i$. Specifically, $v_{ik} = Q_i(x_i(k\veps), \mathbf{0}_{-i})$, where $x_i(k\veps)$ is the optimal effort level player $i$ chooses given allocation $p_i =k\veps$ and zero spillovers.
    \item \textbf{Capacity:} The knapsack capacity is $C = \lfloor 1/\veps \rfloor$.
\end{itemize}
Under this mapping, the NSR problem becomes:
\[
\begin{aligned}
\max \quad & \sum_{i \in \mathcal{N}} \sum_{k=0}^{\lfloor 1/\veps \rfloor} v_{ik}\, z_{ik} \\
\text{s.t.} \quad & \sum_{i \in \mathcal{N}} \sum_{k=0}^{\lfloor 1/\veps \rfloor} k\, z_{ik} \;\le\; C, \\
& \sum_{k=0}^{\lfloor 1/\veps \rfloor} z_{ik} = 1, \quad \forall\, i \in \mathcal{N}, \\
& z_{ik} \in \{0,1\}, \quad \forall\, i \in \mathcal{N}, \;\; k \in \{0, \ldots, \lfloor 1/\veps \rfloor\}.
\end{aligned}
\]
This is exactly the 0-1 MCKP formulation: the constraint $\sum_k z_{ik} = 1$ ensures exactly one allocation level is assigned to each player, and the capacity constraint enforces the budget $\sum_i p_i = \sum_i k_i \veps \le 1$. We can therefore solve the NSR problem using standard MCKP algorithms.

\paragraph{Running time.} Since the problem is equivalent to the 0-1 MCKP, we can solve it exactly using standard dynamic programming approaches. Specifically, the problem can be solved exactly in time $O(n_{items} \cdot C)$, where $n_{items}$ is the total number of items and $C$ is the knapsack capacity \cite{Kellerer2004}. In our case, the number of classes is $m=N$, and the number of items per class is $\lfloor 1/\veps \rfloor + 1$. Thus, the total number of items is $n_{items} = O(N/\veps)$. The dynamic programming algorithm runs in $O(n_{items} \cdot C) = O(\frac{N}{\veps} \cdot \frac{1}{\veps}) = O(\frac{N}{\veps^2})$.
\end{proofof}

\begin{proofof}{lem:price_comp}
If $x_i^{\dagger}=0$ the claim is immediate, since $y_i(\cdot)\ge 0$. Assume $x_i^{\dagger}>0$.
 
\paragraph{Step 1.}
The map $z\mapsto U_i\left((z,\mathbf{x}_{-i});\,p\right)=p\,Q_i(z,\mathbf{x}_{-i})-c_i(z)$ is differentiable
and maximized over $[0,1]$ at the point $z=x_i^{\dagger}\in(0,1]$. A maximizer that is not the left
endpoint has non-negative left-derivative there; therefore
\begin{equation}
\label{eq:foc}
p\,\frac{\partial Q_i}{\partial x_i}(x_i^{\dagger},\mathbf{x}_{-i})\;-\;c_i'(x_i^{\dagger})\;\ge\;0.
\end{equation}
 
\paragraph{Step 2.}
By the marginal condition of Definition~\ref{def:beta_bounded},
$\frac{\partial Q_i}{\partial x_i}(x_i^{\dagger},\mathbf{x}_{-i})\le(1+\eta)\,
\frac{\partial Q_i}{\partial x_i}(x_i^{\dagger},\mathbf{0}_{-i})$.
 
Writing $\tilde p:=(1+\eta)p$ and
combining with \eqref{eq:foc},
\begin{equation}
\label{eq:compensated}
\tilde p\,\frac{\partial Q_i}{\partial x_i}(x_i^{\dagger},\mathbf{0}_{-i})-c_i'(x_i^{\dagger})
=(1+\eta)\,p\,\frac{\partial Q_i}{\partial x_i}(x_i^{\dagger},\mathbf{0}_{-i})-c_i'(x_i^{\dagger})
\;\ge\; p\,\frac{\partial Q_i}{\partial x_i}(x_i^{\dagger},\mathbf{x}_{-i})-c_i'(x_i^{\dagger})\;\ge\;0.
\end{equation}
 
\paragraph{Step 3.}
Let $h(z):=U_i\left((z,\mathbf{0}_{-i});\,\tilde p\right)=\tilde p\,Q_i(z,\mathbf{0}_{-i})-c_i(z)$. As $Q_i$ is
concave in own effort and $c_i$ is convex, $h$ is concave. Hence, its derivative
$h'(z)=\tilde p\,\frac{\partial Q_i}{\partial x_i}(z,\mathbf{0}_{-i})-c_i'(z)$ is non-increasing in $z$.
By \eqref{eq:compensated}, $h'(x_i^{\dagger})\ge 0$, hence $h'(z)\ge 0$ for all $z\le x_i^{\dagger}$;
thus $h$ is non-decreasing on $[0,x_i^{\dagger}]$ and attains a maximum at some point
$\ge x_i^{\dagger}$. Therefore $y_i(\tilde p)=\max\arg\max_z h(z)\ge x_i^{\dagger}$.
\end{proofof}
 
\begin{proofof}{lem:knapsack_abstract}
Define the \emph{density} of item $i$ as $v_i/w_i$ when $w_i>0$, regarding zero-weight items
($w_i = 0$) as having infinite density. List the items in non-increasing order of density (so all
zero-weight items come first). Add items to $S$ in this order, stopping just before the first item
$b$ whose inclusion would make $\sum_{i\in S}w_i$ exceed $1$. 
 
If no such $b$ exists, then $\sum_i w_i\le 1$, so $S=\{1,\dots,N\}$ is feasible and
$\sum_{i\in S}v_i=\sum_i v_i\ge\frac1W\sum_i v_i$, and the claim holds. Otherwise, by construction
$\sum_{i\in S}w_i\le 1$ while $\sum_{i\in S\cup\{b\}}w_i>1$. The items in $S\cup\{b\}$ are those of
highest density; hence the ratio of their total value to their total weight is at least the
corresponding ratio of the entire set:
\[
\frac{\sum_{i\in S\cup\{b\}}v_i}{\sum_{i\in S\cup\{b\}}w_i}\;\ge\;\frac{\sum_i v_i}{\sum_i w_i}.
\]
Using $\sum_{i\in S\cup\{b\}}w_i>1$ and $\sum_i w_i\le W$,
\[
\sum_{i\in S}v_i+v_b\;=\;\sum_{i\in S\cup\{b\}}v_i
\;\ge\;\left(\sum_{i\in S\cup\{b\}}w_i\right)\frac{\sum_i v_i}{\sum_i w_i}
\;>\;\frac1W\sum_i v_i.
\]
Since $v_b\le v_{\max}$, we conclude $\sum_{i\in S}v_i\ge\frac1W\sum_i v_i-v_{\max}$.
\end{proofof}
 

\section{Optimal Mechanism for Trees over $\veps$-discretization}
\label{appendix:trees}

In this appendix, we show that when the spillover topology forms a rooted tree, a near-optimal solution can be computed via dynamic programming. We present an algorithm that handles arbitrary branching factors.

\paragraph{Setting.}
Consider instances where the interaction graph $G=(\mathcal{N},E)$ forms a rooted tree. Each player (except the root) receives spillovers from exactly one parent. Under the graph-based spillover model of Example~\ref{graph example} with linear costs $c_i(x_i) = c_i \cdot x_i \quad (c_i>0)$, each node's quality depends only on its own effort and its parent's effort:
\[
Q_i(x_i, x_{par}) = x_i (q_i + g_{par,i} \cdot x_{par}),
\]
where $q_i > 0$ is player~$i$'s intrinsic quality and $g_{par,i} \geq 0$ is the spillover coefficient from player~$i$'s parent. The utility of player $i$ is
\[
U_i = p_i \cdot Q_i(x_i, x_{par}) - c_i \cdot x_i = x_i \cdot \left[ p_i(q_i + g_{par,i} \cdot x_{par}) - c_i \right].
\]
Since utility is linear in $x_i$, best responses are \emph{binary}: $x_i^\star \in \{0, 1\}$. 

\paragraph{Discretization.}
Let $\veps > 0$ be the granularity parameter. We discretize allocations:
\[
\mathcal{B}(\veps) = \{k\veps : k = 0, 1, \ldots, \lfloor 1/\veps \rfloor\}
\]
with $|\mathcal{B}(\veps)| = O(1/\veps)$. Effort levels take values in $\mathcal{X} = \{0, 1\}$. For any $x \geq 0$, we define the \emph{$\veps$-ceiling} $\lceil x \rceil_\veps := \lceil x / \veps \rceil \cdot \veps$, the smallest multiple of $\veps$ that is at least $x$.

\paragraph{Incentive costs.}
For each node $u$, the \emph{incentive cost} for effort level $x_u = 1$ given parent effort $x_{par}$ is:
\[
\rho_u(x_{par}) = \frac{c_u}{q_u + g_{par,u} \cdot x_{par}},
\]
representing the minimum allocation to make $x_u = 1$ incentive-compatible. We round to the discrete grid:
\[
\hat{\rho}_u(x_{par}) = \lceil \rho_u(x_{par}) \rceil_\veps.
\]
Note that $\hat{\rho}_u(0) \geq \hat{\rho}_u(1)$: positive parent effort reduces the allocation shares needed to incentivize participation.

\paragraph{Key insight.}
The tree structure enables a recursive decomposition. If we fix node $u$'s effort level $x_u$, the children's subtrees become \emph{independent} optimization problems coupled only by the shared allocation constraint. This motivates a bottom-up DP computing, for each node, the optimal total quality achievable in its subtree as a function of available allocation and parent effort.

\paragraph{Value functions.}
For each node $u$ with children $v_1, \ldots, v_m$, we define two value functions. First, the \emph{descendants value}: for $b \in \mathcal{B}(\veps)$ and $x_u \in \{0,1\}$,
\begin{equation} \label{eq:Tu-def}
T_u(b, x_u) \;:=\; \max_{\substack{b_1, \ldots, b_m \in \mathcal{B}(\veps) \\ b_1 + \cdots + b_m \leq b}} \;\sum_{j=1}^{m} V_{v_j}(b_j, x_u),
\end{equation}
i.e., the maximum total quality achievable from $u$'s children's subtrees when the total allocation to those subtrees is at most $b$ and $u$ exerts effort $x_u$. When $u$ is a leaf ($m=0$), $T_u(b, x_u) = 0$ for all $b, x_u$.

Second, the \emph{subtree value}: for $b \in \mathcal{B}(\veps)$ and $x_{par} \in \{0,1\}$,
\begin{equation} \label{eq:Vu-def}
V_u(b, x_{par}) \;:=\; \max_{\substack{x_u \in \{0,1\} \\ \hat{\rho}_u(x_{par}) \cdot x_u \leq b}} \left\{ Q_u(x_u, x_{par}) + T_u\!\left(b - \hat{\rho}_u(x_{par}) \cdot x_u,\; x_u\right) \right\},
\end{equation}
i.e., the maximum total quality in $u$'s entire subtree (including $u$ itself), given allocation budget $b$ and parent effort $x_{par}$.

To extract the optimal solution, we also store optimal decisions $X_u^\star(b, x_{par})$ and $B_u^\star(b, x_u, v)$.

\paragraph{Algorithm overview.}
The algorithm proceeds in two phases. \textit{Phase 1} (Bottom-Up DP) processes nodes from leaves to root, computing value functions via subroutine \textsc{SCBA} (Sequential Children Budget Allocation) for aggregating children. \textit{Phase 2} (Top-Down Extraction) reconstructs the optimal allocation $\bft{p}^\star$.

\begin{algorithm}[t]
\caption{Hierarchical Optimal Payment (HOP)}
\label{alg:HOP}

\begin{algorithmic}[1]
\REQUIRE Tree $G$ with root $r$, intrinsic qualities $(q_u)$, spillover coefficients $(g_{par,u})$, costs $(c_u)$, granularity $\veps$
\ENSURE Optimal allocation $\mathbf{p}^\star \in \mathcal{B}(\veps)^N$

\STATE $\mathcal{B}(\veps) \gets \{k\veps : k = 0, 1, \ldots, \lfloor 1/\veps \rfloor\}$

\FOR{$(u, x_{par}) \in (\mathcal{N} \setminus \{r\}) \times \{0,1\}$}    \STATE $\hat{\rho}_u(x_{par}) \gets \lceil \rho_u(x_{par}) \rceil_\veps$
    \label{line:rho}
\ENDFOR
\STATE $\hat{\rho}_r(0) \gets \lceil \rho_r(0) \rceil_\veps$

\FOR{$u \in G$ in post-order}
\label{line:postorder}

    \FOR{$x_u \in \{0,1\}$}
        \STATE $T_u(\cdot, x_u), B_u^\star(\cdot, x_u, \cdot)
        \gets \text{\textsc{SCBA}}\left(u,\; x_u,\; \{V_v : v \in \text{children}(u)\}\right)$
        \label{line:scba-call}
    \ENDFOR

    \FOR{$(x_{par}, b) \in \{0,1\} \times \mathcal{B}(\veps)$}
    \label{line:vu-start}

        \IF{$b \geq \hat{\rho}_u(x_{par})$}
        \label{line:vu-if}
            \STATE $V_u(b, x_{par}) \gets
            \max\left\{
                T_u(b, 0),\;
                Q_u(1, x_{par}) + T_u\!\left(b - \hat{\rho}_u(x_{par}),\, 1\right)
            \right\}$
            \label{line:vu-max}

            \STATE $X_u^\star(b, x_{par}) \gets \arg\max$ of above
        \ELSE
            \STATE $V_u(b, x_{par}) \gets T_u(b, 0)$
            \STATE $X_u^\star(b, x_{par}) \gets 0$
            \label{line:vu-end}
        \ENDIF

    \ENDFOR
\ENDFOR

\RETURN $\text{\textsc{Extract}}\!\left(r,\; \{X_u^\star\}_{u},\; \{B_u^\star\}_{u},\; \{\hat{\rho}_u\}_{u}\right)$
\label{line:return}

\end{algorithmic}
\end{algorithm}

\begin{algorithm}[t]
\caption{Sequential Children Budget Allocation (SCBA)}
\label{alg:SCBA}

\begin{algorithmic}[1]
\REQUIRE Node $u$, effort level $x_u$, value functions $\{V_v : v \in \text{children}(u)\}$
\ENSURE Value function $T_u(\cdot, x_u)$ and budget decisions $B_u^\star(\cdot, x_u, \cdot)$

\STATE $T_u(b, x_u) \gets 0$ for all $b \in \mathcal{B}(\veps)$
\label{line:scba-init}

\STATE Let $v_1, \ldots, v_m$ be the children of $u$

\FOR{$i = 1, \ldots, m$}
\label{line:scba-child-loop}

    \FOR{$b \in \mathcal{B}(\veps)$ in decreasing order}
    \label{line:scba-budget-loop}

        \STATE $T_u(b, x_u) \gets
        \max_{b' \in \mathcal{B}(\veps):\, b' \leq b}
        \left\{
            T_u(b - b', x_u) + V_{v_i}(b', x_u)
        \right\}$
        \label{line:scba-update}

        \STATE $B_u^\star(b, x_u, v_i) \gets \arg\max$ of above

    \ENDFOR
\ENDFOR

\RETURN $T_u(\cdot, x_u), B_u^\star(\cdot, x_u, \cdot)$

\end{algorithmic}
\end{algorithm}

\begin{algorithm}[t]
\caption{Solution Extraction}
\label{alg:Extract}

\begin{algorithmic}[1]
\REQUIRE Root $r$, decision tables $\{X_u^\star\}_{u \in \mathcal{N}}$, budget tables $\{B_u^\star\}_{u \in \mathcal{N}}$, incentive costs $\{\hat{\rho}_u\}_{u \in \mathcal{N}}$
\ENSURE Optimal allocation $\mathbf{p}^\star = (p_1, \ldots, p_N)$

\STATE $p_u \gets 0$ for all $u \in \mathcal{N}$
\STATE $\text{\textsc{ExtractRec}}(r, \max \mathcal{B}(\veps), 0)$
\RETURN $\mathbf{p}^\star$

\STATE
\STATE \textbf{Function} $\text{\textsc{ExtractRec}}(u, b, x_{par})$

\STATE $x_u \gets X_u^\star(b, x_{par})$
\STATE $p_u \gets \hat{\rho}_u(x_{par}) \cdot x_u$
\STATE $b_{rem} \gets b - p_u$
\STATE Let $v_1, \ldots, v_m$ be the children of $u$

\FOR{$i = m, \ldots, 1$}
    \STATE $b_{v_i} \gets B_u^\star(b_{rem}, x_u, v_i)$
    \STATE $\text{\textsc{ExtractRec}}(v_i, b_{v_i}, x_u)$
    \STATE $b_{rem} \gets b_{rem} - b_{v_i}$
\ENDFOR

\STATE \textbf{End Function}

\end{algorithmic}
\end{algorithm}

\paragraph{Algorithm details.}
Algorithm~\ref{alg:HOP} first pre-computes all incentive costs $\hat{\rho}_u(x_{par})$ for $x_{par} \in \{0,1\}$ (Line~\ref{line:rho}). Then, for each node $u$ in post-order (leaves to root, Line~\ref{line:postorder}), it:
\begin{enumerate}
    \item Calls \textsc{SCBA} (Algorithm~\ref{alg:SCBA}, Line~\ref{line:scba-call}) to compute $T_u(b, x_u)$: the maximum total quality from $u$'s descendants given allocation budget $b$ and effort $x_u$. The value functions $\{V_v\}$ of $u$'s children, already computed in previous post-order iterations, are passed as input. \textsc{SCBA} uses \emph{sequential convolution}--processing children one by one, updating $T_u$ by considering all ways to split the budget between the current child and previously processed children. The budget loop (Line~\ref{line:scba-budget-loop}) iterates in decreasing order to ensure correct in-place updates.
    \item Computes $V_u(b, x_{par})$ (Lines~\ref{line:vu-start}--\ref{line:vu-end} of Algorithm~\ref{alg:HOP}): for each pair $(x_{par}, b)$, if the budget $b$ is at least $\hat{\rho}_u(x_{par})$ (Line~\ref{line:vu-if}), both effort levels are considered--choosing $x_u = 0$ yields value $T_u(b, 0)$, while choosing $x_u = 1$ contributes quality $Q_u(1, x_{par}) = q_u + g_{par,u} \cdot x_{par}$, consumes allocation $\hat{\rho}_u(x_{par})$, and passes positive spillover to descendants. If $b < \hat{\rho}_u(x_{par})$, only $x_u = 0$ is feasible.
\end{enumerate}

After the bottom-up phase, Algorithm~\ref{alg:Extract} reconstructs the optimal allocation via top-down traversal. The decision tables $\{X_u^\star\}$, budget tables $\{B_u^\star\}$, and incentive costs $\{\hat{\rho}_u\}$ are passed as input. It processes children in reverse order (from $v_m$ to $v_1$) to correctly recover the budget splits stored during \textsc{SCBA}.


We now establish the correctness of the algorithm through two lemmas: the first shows that \textsc{SCBA} correctly computes the descendants value $T_u$, and the second uses it to show that Algorithm~\ref{alg:HOP} correctly computes the subtree value $V_u$.

\begin{lemma}
\label{lemma:scba}
Let $u$ be a node with children $v_1, \ldots, v_m$. If $V_{v_j}$ satisfies~\eqref{eq:Vu-def} for every child $v_j$, then upon termination of $\textsc{SCBA}(u, x_u, \{V_v\})$ (Algorithm~\ref{alg:SCBA}), $T_u(\cdot, x_u)$ satisfies~\eqref{eq:Tu-def}.
\end{lemma}

\begin{proof}
We prove by induction on the number of children processed in the outer loop (Line~\ref{line:scba-child-loop} of Algorithm~\ref{alg:SCBA}) that after processing children $v_1, \ldots, v_i$, for every $b \in \mathcal{B}(\veps)$:
\begin{equation} \label{eq:scba-inductive}
T_u(b, x_u) = \max_{\substack{b_1, \ldots, b_i \in \mathcal{B}(\veps) \\ b_1 + \cdots + b_i \leq b}} \sum_{j=1}^{i} V_{v_j}(b_j, x_u).
\end{equation}

\textit{Initialization.} Before any child is processed, $T_u(b, x_u) = 0$ for all $b \in \mathcal{B}(\veps)$ (Line~\ref{line:scba-init}). If $m = 0$ (no children), the claim holds vacuously: the empty sum equals $0$, and the only feasible allocation is the empty one.

\textit{Inductive step.} Suppose~\eqref{eq:scba-inductive} holds for $i-1$. In the $i$-th iteration, the inner loop (Line~\ref{line:scba-budget-loop}) iterates over budget values $b \in \mathcal{B}(\veps)$ in \emph{decreasing} order. For each $b$, Line~\ref{line:scba-update} computes:
\[
T_u(b, x_u) \gets \max_{b' \in \mathcal{B}(\veps):\, b' \leq b} \left\{ T_u(b - b', x_u) + V_{v_i}(b', x_u) \right\}.
\]
Since $b$ is processed in decreasing order, when we evaluate $T_u(b - b', x_u)$ on the right-hand side for any $b' \leq b$, we have $b - b' \leq b$. Since all budget values strictly less than $b$ have not yet been updated in the current ($i$-th) iteration, $T_u(b - b', x_u)$ still holds the value from iteration $i-1$. By the inductive hypothesis, this equals the optimum over $v_1, \ldots, v_{i-1}$ with budget $b - b'$. (When $b' = 0$, we have $V_{v_i}(0, x_u) = 0$, and $T_u(b, x_u)$ retains its previous value, which is also correct.)

Therefore, the update computes:
\[
T_u(b, x_u) = \max_{b' \in \mathcal{B}(\veps):\, b' \leq b} \left\{ \left( \max_{\substack{b_1, \ldots, b_{i-1} \in \mathcal{B}(\veps) \\ \sum_{j=1}^{i-1} b_j \leq b - b'}} \sum_{j=1}^{i-1} V_{v_j}(b_j, x_u) \right) + V_{v_i}(b', x_u) \right\}.
\]
Setting $b_i = b'$, this equals:
\[
T_u(b, x_u) = \max_{\substack{b_1, \ldots, b_i \in \mathcal{B}(\veps) \\ b_1 + \cdots + b_i \leq b}} \sum_{j=1}^{i} V_{v_j}(b_j, x_u),
\]
completing the inductive step. After all $m$ children, $T_u$ satisfies~\eqref{eq:Tu-def}.
\end{proof}

\begin{lemma}
\label{lemma:Vu}
For every node $u \in \mathcal{N}$, the values $V_u(b, x_{par})$ computed in Lines~\ref{line:vu-start}--\ref{line:vu-end} of Algorithm~\ref{alg:HOP} satisfy~\eqref{eq:Vu-def} for all $(b, x_{par}) \in \mathcal{B}(\veps) \times \{0,1\}$.
\end{lemma}

\begin{proof}
We prove by induction on the height of $u$'s subtree.

\textit{Base case.} For a leaf node $u$ (height $0$), node $u$ has no children, so by Lemma~\ref{lemma:scba} applied with $m = 0$, $T_u(b, x_u) = 0$ for all $b$ and $x_u$. In Lines~\ref{line:vu-start}--\ref{line:vu-end}:
\begin{itemize}
    \item If $b \geq \hat{\rho}_u(x_{par})$: Line~\ref{line:vu-max} computes $V_u(b, x_{par}) = \max\{0,\; Q_u(1, x_{par})\} = Q_u(1, x_{par})$, since $Q_u(1, x_{par}) = q_u + g_{par,u} \cdot x_{par} > 0$. This matches~\eqref{eq:Vu-def}: the budget suffices to incentivize $u$, and participation yields positive quality.
    \item If $b < \hat{\rho}_u(x_{par})$: $V_u(b, x_{par}) = 0$. This matches~\eqref{eq:Vu-def}: only $x_u = 0$ is feasible, contributing $Q_u(0, x_{par}) = 0$.
\end{itemize}

\textit{Inductive step.} Let $u$ have subtree height $h > 0$ and children $v_1, \ldots, v_m$. By the inductive hypothesis, $V_{v_j}$ satisfies~\eqref{eq:Vu-def} for every child $v_j$ (each having subtree height at most $h-1$). By Lemma~\ref{lemma:scba}, the call to \textsc{SCBA} at Line~\ref{line:scba-call} yields $T_u$ satisfying~\eqref{eq:Tu-def} for each $x_u \in \{0,1\}$.

Now consider Lines~\ref{line:vu-start}--\ref{line:vu-end}. For any $(x_{par}, b) \in \{0,1\} \times \mathcal{B}(\veps)$:
\begin{itemize}
    \item If $b \geq \hat{\rho}_u(x_{par})$, both effort levels are feasible. Line~\ref{line:vu-max} computes:
    \[
    V_u(b, x_{par}) = \max\!\left\{\underbrace{T_u(b, 0)}_{\text{skip } u},\;\; \underbrace{Q_u(1, x_{par}) + T_u\!\left(b - \hat{\rho}_u(x_{par}),\, 1\right)}_{\text{incentivize } u}\right\}.
    \]
    The first branch corresponds to $x_u = 0$ (zero quality from $u$, full budget to descendants); the second to $x_u = 1$ (quality $Q_u(1, x_{par})$ from $u$, allocation $\hat{\rho}_u(x_{par})$ consumed, remaining budget to descendants). Since $T_u$ satisfies~\eqref{eq:Tu-def}, each branch optimally distributes the available budget among descendants. This matches~\eqref{eq:Vu-def}.

    \item If $b < \hat{\rho}_u(x_{par})$, only $x_u = 0$ is feasible, and $V_u(b, x_{par}) = T_u(b, 0)$, matching~\eqref{eq:Vu-def}.
\end{itemize}
Note that $b - \hat{\rho}_u(x_{par}) \in \mathcal{B}(\veps)$ whenever $b \geq \hat{\rho}_u(x_{par})$, since both are multiples of $\veps$.
\end{proof}

\begin{theorem}
\label{thm:hop}
Algorithm~\ref{alg:HOP} computes an optimal solution to Problem~\eqref{eq:design-problem} over allocations in $\mathcal{B}(\veps)^N$, in time $O(N / \veps^2)$.
\end{theorem}

\begin{proof}
\textbf{Correctness.} By Lemma~\ref{lemma:Vu}, $V_r(\max \mathcal{B}(\veps), 0)$ equals the maximum total quality achievable over all discretized allocations summing to at most $\max \mathcal{B}(\veps) \leq 1$, where the root has no parent (effort $x_{par} = 0$). Algorithm~\ref{alg:Extract} reconstructs the optimal allocation by following the stored argmax decisions (see Remark~\ref{rem:extraction} below).

\textbf{Runtime.} Let $K = |\mathcal{B}(\veps)| = \lfloor 1/\veps \rfloor + 1 = O(1/\veps)$, and let $d_u$ denote the number of children of node $u$.

\textit{Incentive costs.} Computing $\hat{\rho}_u(x_{par})$ for all $(u, x_{par}) \in \mathcal{N} \times \{0,1\}$ requires $2N$ evaluations, each in $O(1)$ time.

\textit{SCBA.} Consider a single call $\textsc{SCBA}(u, x_u, \cdot)$ (Algorithm~\ref{alg:SCBA}). For each child $v_i$, the inner loop iterates over $K$ values of $b$, and for each $b$ computes a maximum over at most $K$ values of $b'$, costing $O(K^2)$ per child. With $d_u$ children, one call costs $O(d_u \cdot K^2)$. Since \textsc{SCBA} is called twice per node (once per $x_u \in \{0,1\}$), the total cost across all \textsc{SCBA} calls is
\[
\sum_{u \in \mathcal{N}} 2 \cdot d_u \cdot K^2 = 2K^2 \sum_{u \in \mathcal{N}} d_u = 2K^2(N-1) = O(N K^2).
\]

\textit{Node optimization.} In Lines~\ref{line:vu-start}--\ref{line:vu-end} of Algorithm~\ref{alg:HOP}, for each node $u$ the loop iterates over $|\{0,1\}| \cdot K = 2K$ pairs $(x_{par}, b)$, each requiring $O(1)$ operations (a single comparison and table lookup). The total cost across all nodes is $2NK = O(NK)$.

\textit{Extraction.} Algorithm~\ref{alg:Extract} visits each node exactly once, performing $O(1)$ operations per node, for a total of $O(N)$.

Summing all contributions: $O(N) + O(NK^2) + O(NK) + O(N) = O(NK^2) = O(N/\veps^2)$.
\end{proof}

\begin{remark}[Extraction Correctness]
\label{rem:extraction}
Algorithm~\ref{alg:Extract} follows stored decisions to reconstruct the optimal allocation. At each node $u$ with budget $b$ and parent effort $x_{par}$:
\begin{enumerate}
    \item It retrieves $x_u = X_u^\star(b, x_{par})$ and sets $p_u = \hat{\rho}_u(x_{par}) \cdot x_u$ (zero if $x_u = 0$).
    \item The remaining budget $b_{rem} = b - p_u$ is distributed among children.
    \item Processing children in reverse order ($v_m, \ldots, v_1$): $B_u^\star(b_{rem}, x_u, v_m)$ gives $v_m$'s allocation when the total children budget is $b_{rem}$; after subtracting $b_{v_m}$, $B_u^\star(b_{rem} - b_{v_m}, x_u, v_{m-1})$ gives $v_{m-1}$'s allocation; and so on.
\end{enumerate}
This exactly reverses the sequential convolution in Algorithm~\ref{alg:SCBA}, recovering the optimal budget split. By construction, the total allocation satisfies $\sum_i p_i \leq \max \mathcal{B}(\veps)$, and the induced total quality equals $V_r(\max \mathcal{B}(\veps), 0)$.
\end{remark}

\begin{remark}
The runtime is independent of the maximum branching factor. A node with many children requires more convolutions, but this is offset by those children having fewer children on average. The total runtime is proportional to the number of edges in any tree, which is ($N-1$).
\end{remark}

\section{Proofs Omitted from Section \ref{dense graph}}

\begin{proofof}{thm random alg}
Our proof proceeds in five parts. First, we formalize the exact allocation selection and lowest costs. Second, we bound the minimal number of players that \text{\stocalg} can incentivize under a clean event. Similarly, in the third step, we bound the maximal number of players the optimal allocation $\bft{p}^\star$ can incentivize under a clean event. In the fourth step, we compare the welfare induced by the active set of \text{\stocalg} and the optimal active set. Lastly, we quantify the probability for the clean event.

\paragraph{Step 1. Formalizing the exact allocations and lowest costs}

We write the proof in terms of the unnormalized quality denominator. Let $\mu = \bar{q} r$. Notice that if $\mu = 1$ then $q^\star = 1$ and therefore $g_{ij} = q_i = r_{i, j} = 1$ for every $i,j \in [N]$, in which case all the qualities are equal and therefore choosing the least-cost players yields the largest incentivized set with the highest welfare. Thus, through the rest of the proof, we analyze for $\mu < 1$. For a
set of active players $S\subseteq[N]$, define
\[
    B_i(S) := q_i+\sum_{j\in S,\,j\ne i} r_{ij}g_{ij}.
\]
If $x^S$ is the binary effort profile that activates exactly the players in
$S$, then
\[
    SW(x^S)=\frac{1}{N}\sum_{i\in S}B_i(S).
\]

For convenience, define
\[
    \Phi(S):=\sum_{i\in S}B_i(S),
    \qquad
    SW(x^S)=\frac{1}{N}\Phi(S).
\]

The factor $1/N$ cancels from the incentive threshold of each player. Indeed, when the active
set is $S$, player~$i$'s utility can be written as
\[
    U_i(x)=\frac{1}{N}x_i\left(p_iB_i(S)-c_i\right).
\]
Thus player $i$ is willing to exert effort in the active set $S$ whenever
\[
    p_iB_i(S)\ge c_i.
\]
Throughout the proof, $\mathbf{\bar x(p)}$ denotes the greatest equilibrium. Hence, if
$p_iB_i(S) \geq c_i$, then we say that player $i$ is incentivized, i.e., $x_i=1$ is selected in the greatest
equilibrium. Since the costs are drawn from a continuous distribution, all
costs are strictly positive with probability one. 

Fix a feasible allocation vector $p$ that incentivizes a set $S$. For every active
player $i\in S$, the incentive constraint is
\[
    p_i\ge \frac{c_i}{B_i(S)}.
\]
Reducing $p_i$ to the tight value $c_i/B_i(S)$ for every $i\in S$ does not
change the induced active set in the greatest equilibrium and does not change
social welfare. Consequently, any feasible active set $S$ must satisfy
\begin{equation}\label{eq:tight-feasibility}
    \sum_{i\in S}\frac{c_i}{B_i(S)}\le 1 .
\end{equation}
Let $c_{(1)}\le c_{(2)}\le\cdots\le c_{(N)}$ denote the ordered costs, and
define
\[
    C_k:=\sum_{\ell=1}^k c_{(\ell)}.
\]
For any set $S$ of size $k$, the total cost of the players in $S$ is at least
$C_k$. This observation will be used when upper bounding the size of the
optimal active set.

\paragraph{Step 2. Lower bound on the size of the active set}

Intuitively, when $N$ is large enough, the numerator and denominator in Inequality~\eqref{eq:tight-feasibility} are close to their expectation values. If we could take the denominator as constant, then the left-hand side would read as $\frac{C_k}{N b}$ where $b$ is a constant. Therefore, the larger $N$ is, the finer the discretization in the denominator is, which means that \text{\stocalg} can choose a larger active set. Our goal in this section is to find a lower bound for the size of the active set as a function of $N$. To formalize this intuition, we start by calculating the concentration bounds on the costs and $B_i$.

\begin{lemma}\label{lem:uniform-costs-formal}
Let $c_1,\ldots,c_N\sim\mathrm{Uni}([0,1])$ independently, and let
$c_{(1)}\le\cdots\le c_{(N)}$ denote the order statistics. Define
$C_k=\sum_{\ell=1}^k c_{(\ell)}$. Then, with probability at least
$1-2e^{-2\sqrt N}$, for every $k\in[N]$,
\[
    \left|C_k-\frac{k^2}{2N}\right|
    \le 2N^{3/4}.
\]
\end{lemma}

\begin{lemma}\label{lem:gcs-prefix-formal}
Let $S$ be the fixed set, chosen independently from $(r_i)_i, (q_i)_i$. For all sufficiently large $N$, with
probability at least $1-2N\exp\left(-\frac{\sqrt N}{2(q^\star)^2}\right)$, for every player $i \in S$ it holds that
\begin{align*}
B_i(S)\ge \mu k -2N^{3/4}.
\end{align*}
\end{lemma}

We denote the event under which \Cref{lem:uniform-costs-formal} as $\mathcal{E}_c$. Therefore, the probability of this event satisfies that 
\begin{align*}
Pr(\mathcal{E}_c) \geq 1-2e^{-2\sqrt N}
\end{align*}

Similarly, we denote the event under which \Cref{lem:gcs-prefix-formal} is satisfied as $\mathcal{E}_q$. Therefore, the probability of this event satisfies that
\begin{align*}
Pr(\mathcal{E}_q) \geq  1-2Ne^{-\frac{\sqrt N}{2(q^\star)^2}}
\end{align*}

From \Cref{lem:uniform-costs-formal} and \Cref{lem:gcs-prefix-formal}, for any feasible set $S$ of size $k$ incentivized by \text{\stocalg}, we get that
\begin{align*}
\sum_{i\in S}\frac{c_i}{B_i(S)} \leq \frac{C_k}{\mu k - 2N^{\nicefrac{3}{4}}} \leq \frac{\frac{k^2}{2N} + 2N^{\nicefrac{3}{4}}}{\mu k - 2N^{\nicefrac{3}{4}}}
\end{align*}

The right-hand side represents the smallest active set the algorithm can incentivize, as it upper bounds the costs with the lower bounds of the qualities, increasing in the expression $\sum_i p_i$.

Denote $\alpha = \frac{k}{N}$. Then, a sufficient condition to satisfy Inequality~\eqref{eq:tight-feasibility} is that
\begin{align} \label{ineq alpha lowerbound}
\frac{\frac{1}{2}\alpha^2 + 2N^{-\nicefrac{1}{4}}}{\mu \alpha - 2N^{-\nicefrac{1}{4}}} \leq 1.
\end{align}

Equivalently
\begin{align} \label{ineq alpha condition}
\frac{1}{2}\alpha^2 - \mu \alpha + 4N^{-\nicefrac{1}{4}} \leq 0.
\end{align}

For $\alpha = 2\mu - \frac{A}{\mu}N^{-\nicefrac{1}{4}}$, we get that $\frac{\alpha}{2} - \mu = -\frac{A}{2\mu}N^{-\nicefrac{1}{4}}$. Plugging that into Inequality~\eqref{ineq alpha condition} results in
\begin{align*}
\frac{1}{2}\alpha^2 - \mu \alpha + 4N^{-\nicefrac{1}{4}} = \alpha \left(\frac{\alpha}{2} - \mu \right) + 4N^{-\nicefrac{1}{4}} = -\alpha \frac{A}{2\mu}N^{-\nicefrac{1}{4}} + 4N^{-\nicefrac{1}{4}}
\end{align*}

For $\alpha > \mu$, our sufficient condition becomes
\begin{align*}
\frac{1}{2}\alpha^2 - \mu \alpha + 4N^{-\nicefrac{1}{4}} \leq N^{-\nicefrac{1}{4}} \left( -\frac{A}{2} + 4 \right) \leq 0
\end{align*}
Hence, this condition is satisfied for $A = 10$.
Next, let $f(\alpha) = \frac{1}{2}\alpha^2 - \mu \alpha$. Then $\frac{df}{d\alpha} = \alpha - \mu > 0$, i.e., $f(\alpha)$ is increasing for every $\alpha > \mu$. 
Let $\alpha_0 = \min \{ 2\mu, 1 \}$ and denote 
\[
\alpha_{-} = \alpha_0 - \frac{10}{\mu} N^{-\frac{1}{4}}.
\]

Observe that $\alpha_{-} \leq \alpha = 2\mu - \frac{10}{\mu}N^{-\frac{1}{4}}$, and consider the term $\alpha_- - \mu$. For $\mu < 0.5$ it holds that $\alpha_- - \mu > 0$ for every $N > \left(\frac{10}{\mu^2}\right)^4$. On the other hand, for $\mu \in [0.5, 1)$, we get that $\alpha_- - \mu = 1 - \mu - \frac{10}{\mu}N^{-1/4} > 0$ for every $N > \left( \frac{10}{\mu(1-\mu)} \right)^4$. In other words, for every $\mu \in [0, 1)$, for sufficiently large $N$ it holds that $f(\alpha_{-}) \leq f(\alpha)$, which means that $\alpha_{-}$ satisfies Inequality~\eqref{ineq alpha condition}.

Therefore, the following is a lower bound for $k^a$:
\begin{align} \label{k lower bound}
k_- = \left\lfloor N\left(\alpha_0 -\frac{10}{\mu}N^{-1/4}\right)\right\rfloor.
\end{align}

\paragraph{Step 3. Upper bound on the size of the optimal active set}
Let $m = \ceil{\sqrt{N}}$. Then if the optimal active set size $k^\star$ satisfies that $k^\star < m$, it means that
\begin{align*}
    \Phi(S^\star) \leq k^\star \left(q^\star + k^\star q^\star \right) = O(\left(k^\star \right)^2) = O(N).
\end{align*}

On the other hand, for active set from Algorithm~\text{\stocalg}, we have a lower bound on $B_i$ and $k^a$, that is
\begin{align*}
\Phi(S^a) \geq k_- \left( \mu k_- - 2N^{\nicefrac{3}{4}} \right) = \Omega(N^2)
\end{align*}

Hence, the welfare guarantee is trivial for sufficiently large N. We therefore assume $k^\star \geq m$. To find an upper bound for $k^\star$, we return to the feasibility condition in Inequality~\eqref{eq:tight-feasibility}. The optimal active set is chosen based on the complete information of all the qualities and costs. Therefore, it may include multiple players whose qualities are substantially higher than their expectation. Hence, our goal now is to show that the optimal active set cannot include a large set of those players.

For $T \subseteq S \subseteq N$, let
\[
W(T, S) := \sum_{i \in T} \sum_{j \in S, j \neq i} r_{ij} g_{ij}.
\]
$W(T,S)$ describes the total spillovers of every player in subset $T$ from the players in $S$.

We get an upper bound for $W(T,S)$ using the following lemma.

\begin{lemma} \label{lem:uniform-row-upper-formal}
Let $m=\lceil\sqrt N\rceil$. With probability at least
\[
    1-e^{-\frac{2t \sqrt{N}}{(q^\star)^2}+N\log 3},
\]
it holds that, for every $T\subseteq S\subseteq[N]$ with $|T|\ge m$,
\[
    \sum_{i\in T}\sum_{j\in S,\,j\ne i} r_{ij}g_{ij}
    \le
    \mu |T||S|+|T|N^{3/4}.
\]
\end{lemma}

\Cref{lem:uniform-row-upper-formal} implies that 

\begin{lemma} \label{lem:high quality set size}
Let $T \subseteq S^\star$ such that $B_i(S^\star) > q^\star + \mu k^\star + N^{\nicefrac{3}{4}}$ for every $i \in T$. Then it holds that $|T| \leq m$.
\end{lemma}

Therefore, going back to the feasibility condition, the largest active set is when the denominator is maximized. This happens when $T$ is maximized. Therefore, we compute a bound over the feasibility condition when $T = \sqrt{N}$. In that case, We denote $\mathcal{E}_w$ the event of \Cref{lem:uniform-row-upper-formal}. Therefore, the probability of this event satisfies that
\begin{align*}
Pr(\mathcal{E}_w) \geq 1- e^{-\left(\frac{2}{(q^\star)^2}-\log 3\right)N}
\end{align*}

Let $G^\star\subseteq S^\star$ be the set of players in $S^\star$ that satisfies 
\[
B_i(S^\star) \leq q^\star + \mu k^\star + N^{\nicefrac{3}{4}}.
\]
Then from \Cref{lem:high quality set size}, we know that there are at least $k^\star - m$ such players.
Using \Cref{lem:high quality set size} and \Cref{lem:uniform-costs-formal} in our feasibility condition gives us that
\[
    1 \ge
    \sum_{i\in S^\star}\frac{c_i}{B_i(S^\star)}
    \ge \sum_{i\in G^\star}\frac{c_i}{B_i(S^\star)}
    \ge \frac{\sum_{i\in G^\star}c_i}{q^\star+\mu k^\star+N^{3/4}}
    \ge \frac{C_{k^\star-m}}{q^\star+\mu k^\star+N^{3/4}} \geq \frac{\frac{(k^\star-m)^2}{2N} - 3N^{3/4}}{q^\star+\mu k^\star+N^{3/4}}.
\]

Equivalently,
\begin{align*}
\frac{(k^\star-m)^2}{2N} - 3N^{3/4} \leq q^\star+\mu k^\star+N^{3/4}.
\end{align*}

Dividing by $N$ results in
\begin{align*}
\frac{\left(\alpha^\star - \frac{m}{N}\right)^2}{2}  \leq q^\star N^{-1} + \mu \alpha^\star + 4 N^{-1/4}
\end{align*}

Since $\frac{m}{N} = O(N^{-\nicefrac{1}{2}})$, it holds that
\begin{align*}
(\alpha^{\star})^2 \leq 2 \mu \alpha^\star + O(N^{-1/4}).
\end{align*}

Since $\mu>0$ is fixed, there exists a constant $A_\star=A_\star(\mu,q^\star)>0$
such that, for all sufficiently large $N$,
\begin{equation}
\alpha^\star \leq 2\mu + A_\star N^{-1/4}.
\end{equation}

Observe that if $\mu > \frac{1}{2}$, we get that $\alpha_0 = 1$. Since $\alpha^\star \in [0, 1]$ we can trivially bound it by
\begin{equation} \label{eq:kstar-bound}
\alpha^\star \leq \alpha_0 + A_\star N^{-1/4}.
\end{equation}

Equivalently,
\[
    k^\star
    \le
    N\left(2\mu+A_\star N^{-1/4}\right).
\]

\paragraph{Step 4. Comparing the welfares}
First, let $S_-$ be the set corresponding to $k_-$. Then, since Algorithm~\text{\stocalg} incentivizes the players with the least costs first, it holds that $S_- \subseteq S^a$. Using \Cref{lem:gcs-prefix-formal} results in
\begin{align*}
\Phi(S^a) &\geq \Phi(S_-) \geq k_- \left( \mu k_- - 2N^{3/4} \right)
\end{align*}

Hence,
\begin{align*}
\frac{\Phi(S^a)}{N^2} &\geq \mu k_-^2 N^{-2} - 2k_- N^{-5/4}
\end{align*}
Notice that for sufficiently large $N$, there exists $A_-$ such that $k_- > N(2\mu - A_ N^{-1/4})$. Therefore, there exists $A_L$ such that
\begin{align*}
\frac{\Phi(S^a)}{N^2} &\geq \mu k_-^2 N^{-2} - 2k_- N^{-5/4} \geq \mu (\alpha_0 - A_ N^{-1/4})^2 -2 (\alpha_0 - A_ N^{-1/4}) N^{-1/4} \geq \mu \alpha_0^2  - A_L N^{-1/4}
\end{align*}

On the other hand, Using \Cref{lem:uniform-row-upper-formal}, the non-normalized welfare from the optimal active set satisfies that
\begin{align*}
\Phi(S^\star) \leq k^\star q^\star + \mu (k^\star)^2 + k^\star N^{3/4}
\end{align*}

For sufficiently large $N$, there exists $A^u$ such that
\begin{align*}
\frac{\Phi(S^\star)}{N^2} \leq \frac{k^\star}{N^2} \left( q^\star + N^{3/4} \right) + \mu \frac{(k^\star)^2}{N^2} = (q^\star + N^{3/4}) \frac{\alpha_0 + A_\star N^{-1/4}}{N} + \mu \left( \alpha_0 + A_\star N^{-1/4} \right)^2 \leq  \mu \alpha_0^2 + A_u N^{-1/4}.
\end{align*}

Therefore, the welfare is given by
\begin{align*}
\frac{SW(\bar{\bft{x}}(\bft{p}^a))}{SW(\bar{\bft{x}}(\bft{p}^\star))} \geq \frac{\Phi(S^a)}{ \Phi(S^\star)} \geq \frac{\mu \alpha_0^2 - A_L N^{-1/4}}{\mu \alpha_0^2 + A_u N^{-1/4}}
\end{align*}
Hence, for sufficiently large $N$, there exists constant $C$ such that
\begin{align*}
\frac{SW(\bar{\bft{x}}(\bft{p}^a))}{SW(\bar{\bft{x}}(\bft{p}^\star))} \geq (1-CN^{-1/4}).
\end{align*}

\paragraph{Probability of clean event}
The clean event satisfies that
\begin{align*}
Pr(\mathcal{E}) &= Pr(\mathcal{E}_c \cap \mathcal{E}_q \cap \mathcal{E}_w) = 1 - Pr(\mathcal{E}_c^c \cup \mathcal{E}_q^c \cup \mathcal{E}_w^c)
\end{align*}

From the union bound, we have that
\begin{align*}
Pr(\mathcal{E}_c^c \cup \mathcal{E}_q^c \cup \mathcal{E}_w^c) \leq Pr(\mathcal{E}_c^c) + Pr(\mathcal{E}_q^c) + Pr(\mathcal{E}_w^c).
\end{align*}

For
\begin{align*}
Pr(\mathcal{E}_c) \geq 1-2e^{-2\sqrt N}, \qquad Pr(\mathcal{E}_q) \geq  1-2Ne^{-\frac{\sqrt N}{2(q^\star)^2}}, \qquad Pr(\mathcal{E}_w) \geq 1- e^{-\left(\frac{2}{(q^\star)^2}-\log 3\right)N},
\end{align*}
it holds that
\begin{align*}
Pr(\mathcal{E}_c^c) \leq 2e^{-2\sqrt N}, \qquad Pr(\mathcal{E}_q^c) \leq  2Ne^{-\frac{\sqrt N}{2(q^\star)^2}}, \qquad Pr(\mathcal{E}_w^c) \leq e^{-\left(\frac{2}{(q^\star)^2}-\log 3\right)N}.
\end{align*}
Plugging everything into $Pr(\mathcal{E})$ results in
\begin{align*}
Pr(\mathcal{E}) \geq 1 - 2e^{-2\sqrt N} - 2Ne^{-\frac{\sqrt N}{2(q^\star)^2}} - e^{-\left(\frac{2}{(q^\star)^2}-\log 3\right)N}
\end{align*}

Furthermore, for sufficiently large N, notice that
\[
Ne^{-\frac{\sqrt{N}}{2}} < e^{-\sqrt{N}/4}.
\]

Recall that $q^\star \in (0, 1]$. Hence, for sufficiently large $N$ it holds that
\begin{align*}
Pr(\mathcal{E}) \geq 1 - 5 e^{-\sqrt{N}/4}.    
\end{align*}

\paragraph{Runtime}
Line~\ref{stocalg: sort players} runs in $O(N\log N)$. The loop in
Line~\ref{stocalg: forloop} iterates over at most $N$ values of $k$. Evaluated
directly, Line~\ref{stocalg: incentivize players} costs $O(k)$ per active
player and hence $O(N^3)$ overall. However, we can maintain the denominators
incrementally. Let $B_i^{(k)} = q_i + \sum_{j \le k,\, j \neq i} r_{ij}g_{ij}$.
Computing $B_i^{(N)}$ for all $i \in [N]$ takes $O(N^2)$ once and, when $k$
decreases, each denominator changes by a single term,
$B_i^{(k-1)} = B_i^{(k)} - r_{ik}g_{ik}$. Hence, each iteration performs three
$O(k)$ operations: it updates the $k$ denominators, recomputes the allocations
$p_i = c_i / B_i^{(k)}$, and recomputes their sum for
Line~\ref{stocalg: check implementable}. Summing over the iterations, the loop
costs $O(N^2)$. The total runtime is therefore $O(N^2)$.
This concludes the proof of \Cref{thm random alg}.

\end{proofof}

\begin{corollary} \label{cor: random welfare}
The optimal social welfare value satisfies that
\begin{align*}
\lim_{N \rightarrow \infty} \frac{SW(\bar{\bft{x}}(\bft{p}^\star))}{N} = \bar{q} r \min\{ \left(2\bar{q} r \right)^2, 1 \}.
\end{align*}
\end{corollary}

\begin{proofof}{lem:uniform-costs-formal}
Let
\[
    F_N(t):=\frac{1}{N}\sum_{i=1}^N \mathbf 1\{c_i\le t\}
\]
be the empirical CDF. By the Dvoretzky-Kiefer-Wolfowitz inequality, with
probability at least $1-2e^{-2\sqrt N}$,
\begin{equation}
    \sup_{t\in[0,1]}|F_N(t)-t|\le N^{-1/4}.
    \label{eq:dkw-event}
\end{equation}
Condition on this event. Since the distribution is continuous, the order
statistics are distinct with probability one. For every $\ell\in[N]$,
$F_N(c_{(\ell)})=\ell/N$. Using \eqref{eq:dkw-event} at $t=c_{(\ell)}$ gives
\[
    \frac{\ell}{N}-N^{-1/4}
    \le
    c_{(\ell)}
    \le
    \frac{\ell}{N}+N^{-1/4}.
\]
Summing over $\ell=1,\ldots,k$ yields
\[
    \frac{k(k+1)}{2N}-kN^{-1/4}
    \le
    C_k
    \le
    \frac{k(k+1)}{2N}+kN^{-1/4}.
\]
Because $k\le N$, the difference between $k(k+1)/(2N)$ and $k^2/(2N)$ is at
most $1/2$, and $kN^{-1/4}\le N^{3/4}$. Thus, for all $k\in[N]$,
\[
    \left|C_k-\frac{k^2}{2N}\right|
    \le 2N^{3/4}
\]
for all sufficiently large $N$. This concludes the proof of \Cref{lem:uniform-costs-formal}.
\end{proofof}

\begin{proofof}{lem:gcs-prefix-formal}
For any fixed set $S$, independent of the spillover variables $r_{ij}$, $g_{ij}$, and for every fixed $i\in S_{k_-}$, the sum $\sum_{j\in S_{k_-},\,j\ne i} r_{ij}g_{ij} $
is a sum of $k-1$ independent random variables, each lying in $[0,q^\star]$
and each having expectation $\mu$. By Hoeffding's inequality,
\begin{align*}
\Pr\left[ \sum_{j\in S_{k},\,j\ne i} r_{ij}g_{ij} < \mu(k-1)-N^{3/4} \right] \le e^{-\frac{2N^{3/2}}{(k-1)(q^\star)^2}}.
\end{align*}

Observe that $k \leq N$. Therefore the right-hand side can be upper bounded by
\begin{align*}
e^{-\frac{2N^{3/2}}{(k-1)(q^\star)^2}} \leq e^{-\frac{2\sqrt{N}}{(q^\star)^2}} \leq e^{-\frac{\sqrt{N}}{2(q^\star)^2}}
\end{align*}

For all sufficiently large $N$, $k_-\le N$, and therefore the last term is at
most $\exp(-2\sqrt N/(q^\star)^2)$, which is bounded above by
$2\exp(-\sqrt N/(2(q^\star)^2))$. Also, for all sufficiently large $N$,
\begin{align*}
 \mu(k-1)-N^{3/4} \ge \mu k -2N^{3/4}.
\end{align*}
Since $q_i\ge0$, the same lower bound applies to $B_i(S)$. A union bound
over at most $N$ players gives the claim.
This concludes the proof of \Cref{lem:gcs-prefix-formal}.
\end{proofof}

\begin{proofof}{lem:uniform-row-upper-formal}
Fix $T\subseteq S\subseteq[N]$ and write $t=|T|$, $k=|S|$. The random variable
\[
    W(T,S)=\sum_{i\in T}\sum_{j\in S,\,j\ne i} r_{ij}g_{ij}
\]
is a sum of at most $tk$ independent random variables, each lying in
$[0,q^\star]$. Its expectation is at most $\mu tk$. By Hoeffding's inequality,
\begin{align*}
\Pr\left[W(T,S)>\mu tk+tN^{3/4}\right] \le e^{-\frac{2t^2N^{3/2}}{tk(q^\star)^2}} = e^{-\frac{2tN^{3/2}}{k(q^\star)^2}}.
\end{align*}
    
Since $k\le N$, we can bound this probability by
\begin{align*}
e^{-\frac{2tN^{3/2}}{k(q^\star)^2}} \leq e^{-\frac{2t\sqrt{N}}{(q^\star)^2}}.
\end{align*}

Next, we need to count the possible options to forming subsets of size $t$ from subsets of size $k$, and subset of size $k$ from a set of $N$ players. That is given by
\begin{align*}
\binom{N}{k} \binom{k}{t} \leq \binom{N}{k} \sum_{t = 0}^k \binom{k}{t} \leq \sum_{k = 0}^N \binom{N}{k} \sum_{t = 0}^k \binom{k}{t} = \sum_{k = 0}^N \binom{N}{k} 2^k = 3^N
\end{align*}
    
Therefore, A union bound gives
failure probability at most
\[
    3^N e^{-\frac{2t\sqrt{N}}{(q^\star)^2}} =
    e^{-\frac{2t \sqrt{N}}{(q^\star)^2}+N\log 3}.
\]
This concludes the proof of \Cref{lem:uniform-row-upper-formal}.
\end{proofof}

\begin{proofof}{lem:high quality set size}
Assume toward contradiction that $|T| > m$. Then it holds that
\begin{align*}
\sum_{i \in T} \sum_{j \in S^\star, j\neq i} r_{ij}g_{ij} = \sum_{i \in T} B_i(S^\star) - q_i  > |T| \left( q^\star + \mu k^\star + N^{\nicefrac{3}{4}} \right) - \sum_{i \in T} q_i \geq |T| \left( \mu k^\star + N^{\nicefrac{3}{4}} \right),
\end{align*}
which contradicts \Cref{lem:uniform-row-upper-formal}. This concludes the proof of \Cref{lem:high quality set size}.
\end{proofof}

\section{Experimental Validation} \label{sec:experiments}
In this section, we validate the performance of \text{\stocalg} through simulations on random instances drawn from the setting defined above. We use a fixed random seed to ensure reproducibility across experiments.

\paragraph{Algorithmic implementation}
The implementation uses standard Python scientific computing libraries (NumPy, Matplotlib) and is CPU-based. We consider two algorithms:
\begin{itemize}
    \item \textsc{Greedy Cost Selection} We closely follow the implementation in Algorithm~\ref{find-portions}.
    
    \item \textsc{Equal Allocation} (Baseline): This baseline assigns equal allocation shares $p_i = 1/N$ to all players. We compute the greatest pure Nash equilibrium via best-response dynamics starting from the maximal profile $\bft{x} = (1, \ldots, 1)$. Empirically, the dynamics converge monotonically to the equilibrium within a small number of iterations.
\end{itemize}

\paragraph{Simulation pipeline}
For each parameter configuration that we define, we execute the following procedure:
\begin{enumerate}
    \item Draw 1000 independent game instances with $q_i, g_{ij} \sim \text{Uni}(0, q^\star)$, costs $c_i \sim \text{Uni}(0,1)$ and $r_{ij} \sim \text{Bern}(r)$.
    \item For each instance, we run both \textsc{Greedy Cost Selection} and \textsc{Equal Allocation}.
    \item Record the social welfare achieved and the number of active players $K(\bft{x})$ for each algorithm.
    \item We present the mean values across all instances, with shaded regions indicating  error bars set at three standard deviations.
\end{enumerate}

\paragraph{Experimental setup}
We conducted three sets of experiments:
\begin{enumerate}
    \item \textbf{Varying $N$:} Fix $q^\star = 1$ and vary $N \in \{100, 200, \ldots, 1000\}$ for $r \in \{0.2, 0.5, 0.8\}$.
    \item \textbf{Varying $r$:} Fix $q^\star = 1$ and vary $r \in \{0.05, 0.10, \ldots, 0.95\}$ for $N \in \{100, 500, 1000\}$.
    \item \textbf{Varying $q^\star$:} Fix $N = 1000$, $r = 0.5$, and vary $q^\star \in \{0.05, 0.10, \ldots, 1.0\}$.
\end{enumerate}
Each plot also includes the reference value $4N(\bar{q} r)^3$ from \Cref{cor: random welfare} ($=N(q^\star r)^3/2$, since $\bar{q} = \frac{q^\star}{2}$ under the uniform distribution) and the theoretical prediction $K(\bar{\bft{x}}) \approx2\bar{q} r N (= r q^\star N)$ for the number of incentivized players.

\paragraph{Computational resources}
All experiments were executed on a standard laptop. The entire simulation required approximately 5 hours. 

\subsection{Results}

\begin{figure}[t]
    \centering
    \includegraphics[width=0.9\linewidth]{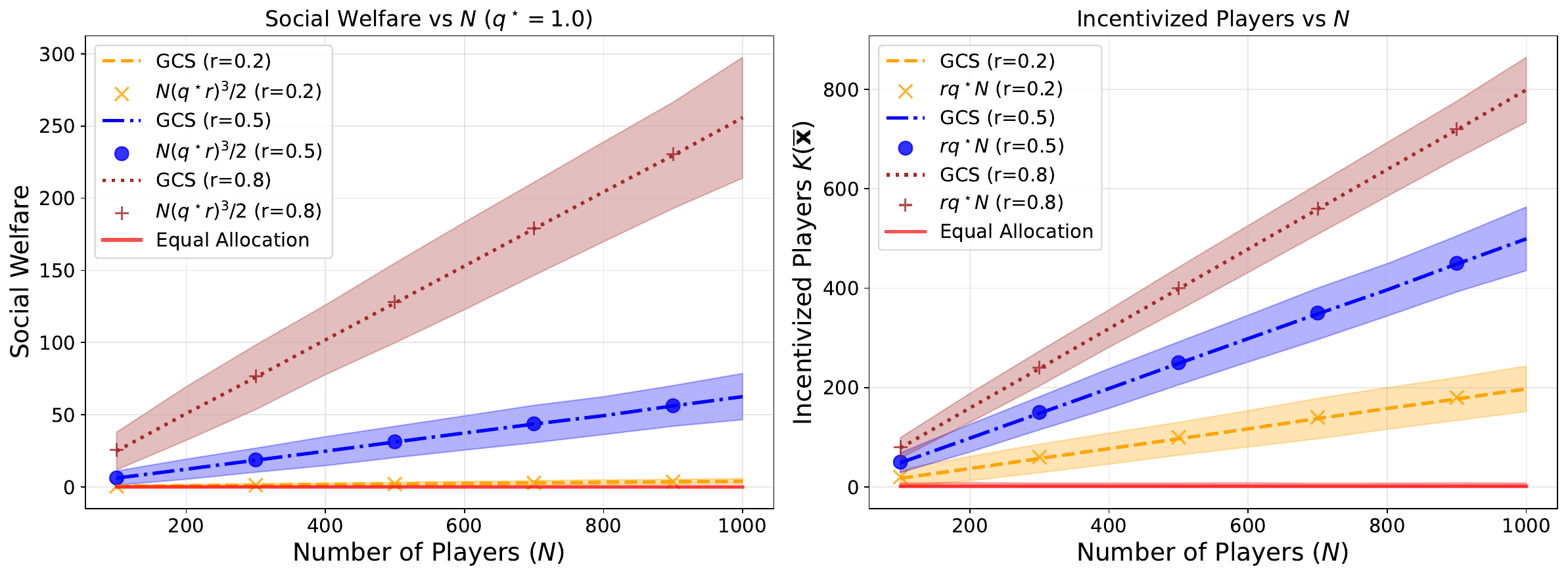}
    \caption{Results for varying $N$ with $q^\star = 1$. Left: Social welfare scales linearly in $N$. Right: Number of incentivized players grows proportionally to $N$. Markers indicate theoretical predictions. Equal Allocation shown for $r=0.8$.}
    \label{fig:exp1}
\end{figure}

\begin{figure}[t]
    \centering
\includegraphics[width=0.9\linewidth]{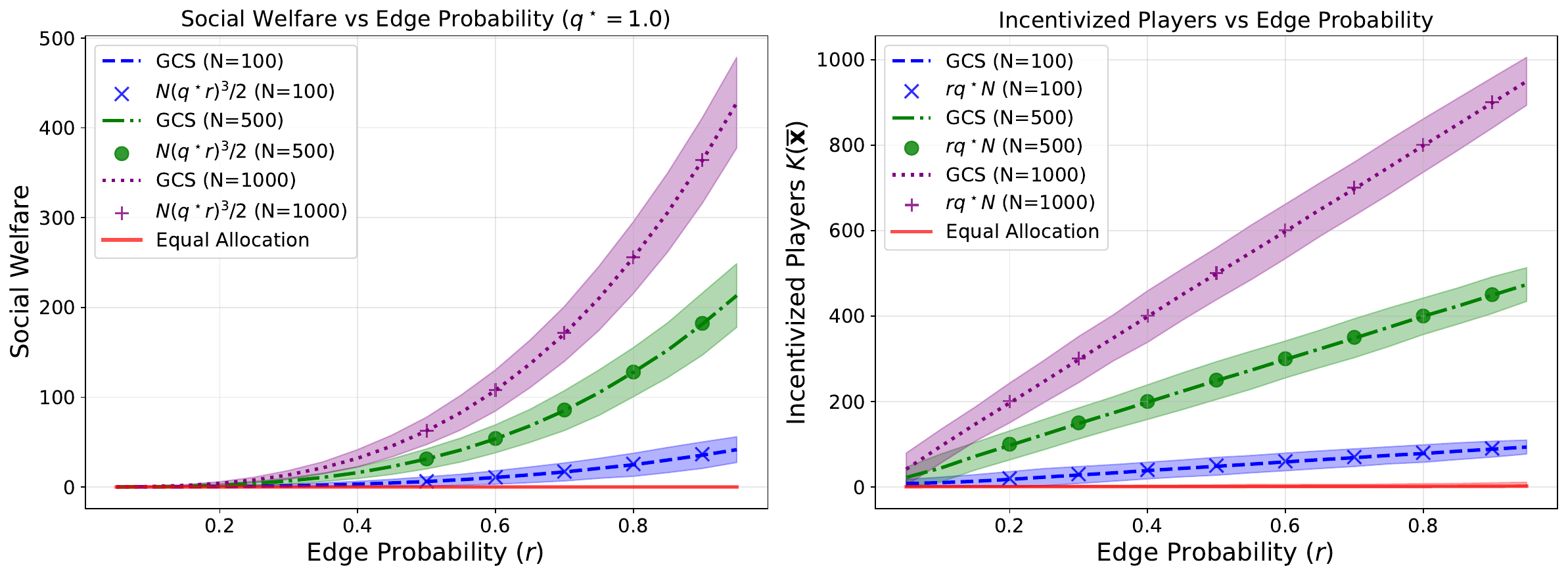}
    \caption{Results for varying $r$ with $q^\star = 1$. Left: Social welfare. Right: Number of active players. Markers indicate the theoretical predictions $N(q^\star r)^3/2$ and $rq^\star N$, respectively. Equal Allocation shown for $N=1000$.}
\end{figure}

\begin{figure}[t]
    \centering
    \includegraphics[width=0.9\linewidth]{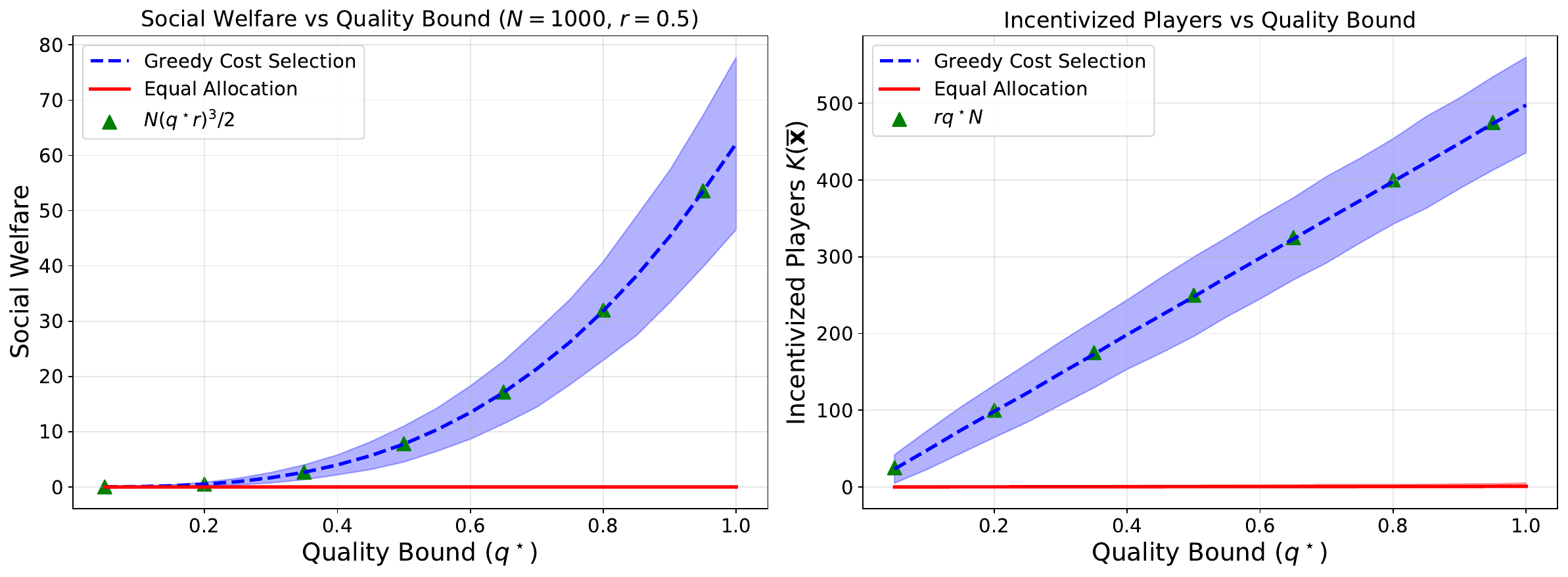}
    \caption{Results for varying $q^\star$ with $N = 1000$ and $r = 0.5$. Left: Social welfare increases with $q^\star$. Right: Larger $q^\star$ enables larger incentivized sets.}
    \label{fig:exp3}
\end{figure}

Figures~\ref{fig:exp1}--\ref{fig:exp3} present the results across all three parameter regimes. We report

\paragraph{Performance of Greedy Cost Selection}
Across all experiments, \textsc{Greedy Cost Selection} achieves average welfare that closely tracks $N(q^\star r)^3/2$, empirically validating Theorem~\ref{thm random alg} for small values of $N$. 

\paragraph{Comparison with Equal Allocation}
The \textsc{Equal Allocation} baseline achieves substantially lower welfare across all parameter configurations. The performance gap widens as $r$ and $q^\star$ increase, demonstrating the importance of cost-aware allocation: uniform fails to concentrate the allocation on players with the highest marginal welfare contribution.

\paragraph{Active player dynamics}
The right panels of Figures~\ref{fig:exp1}--\ref{fig:exp3} display the number of active players $K(\bft{x})$ selected by each algorithm. \textsc{Greedy Cost Selection} achieves incentivized sets of size approximately $K(\bft{x}) \approx r q^\star N$, matching the theoretical prediction. In contrast, \textsc{Equal Allocation} often activates significantly fewer players, as the uniform allocation cannot provide sufficient incentives for high-cost players to participate, even when their participation would be socially beneficial due to spillover effects.

}{\fi}}

\end{document}